\documentclass[twoside]{article}

\usepackage[accepted]{aistats2025}

\usepackage[utf8]{inputenc} 
\usepackage[T1]{fontenc}    
\usepackage{hyperref}       
\usepackage{url}            
\usepackage{booktabs}       
\usepackage{amsfonts}       
\usepackage{nicefrac}       
\usepackage{microtype}      
\usepackage{xcolor}         
\usepackage{tex_macros}
\usepackage{algorithm2e}
\RestyleAlgo{ruled}
\SetKwComment{Comment}{/* }{ */}
\usepackage{tex_macros}

\usepackage{multicol}
\usepackage{wrapfig}
\usepackage{multirow}
\usepackage{array}
\usepackage{wrapfig}
\newcolumntype{L}{>{\centering\arraybackslash}m{0.25\linewidth}}

\begin{document}

%
\runningtitle{Sampling in High-Dimensions using Stochastic Interpolants and FBSDEs}

%

\twocolumn[

\aistatstitle{Sampling in High-Dimensions using Stochastic Interpolants and\\ Forward-Backward Stochastic Differential Equations}

\aistatsauthor{ Anand Jerry George
\And
Nicolas Macris}

\aistatsaddress{ École Polytechnique Fédérale de Lausanne (EPFL),\\
Lab for Statistical Mechanics of Inference in Large Systems (SMILS),\\ CH-1015 Lausanne,\\ Switzerland} ]

\begin{abstract}
  We present a class of diffusion-based algorithms to draw samples from high-dimensional probability distributions given their unnormalized densities. Ideally, our methods can transport samples from a Gaussian distribution to a specified target distribution in finite time. Our approach relies on the stochastic interpolants framework to define a time-indexed collection of probability densities that bridge a Gaussian distribution to the target distribution. Subsequently, we derive a diffusion process that obeys the aforementioned probability density at each time instant. Obtaining such a diffusion process involves solving certain Hamilton-Jacobi-Bellman PDEs. We solve these PDEs using the theory of forward-backward stochastic differential equations (FBSDE) together with machine learning-based methods. Through numerical experiments, we demonstrate that our algorithm can effectively draw samples from distributions that conventional methods struggle to handle.
\end{abstract}

\section{Introduction}

Probabilistic modeling is at the heart of modern machine learning, where data is often conceptualized as samples drawn from a probability distribution in a high-dimensional space. 
In this context, the field of generative modeling, which focuses on drawing new samples from probability distributions given a few samples from the distribution, has witnessed spectacular advancements in recent years, empowering researchers to create synthetic images, videos, and other data with astonishing quality. A concept that has catalyzed significant progress in generative modeling is score-based diffusion modeling \cite{song_generative_2019,song_score-based_2021,ho_denoising_2020}. This technique incrementally introduces noise to the data until it reaches a state of pure noise, all while learning how to reverse this process effectively. 

These new ideas have found applications in the classical sampling problem, wherein the objective is to generate samples from a probability distribution given only its unnormalized density \cite{zhang_path_2022,berner_optimal_2023,vargas_denoising_2022,richter_improved_2023,grenioux_stochastic_2024,huang_reverse_2023}. 
In the realm of computational science and statistics, sampling techniques play a pivotal role in various applications ranging from Bayesian inference to machine learning algorithms. The ability to efficiently sample from complex probability distributions underpins the success of numerous computational methodologies, and traditional sampling methods, such as Markov Chain Monte Carlo (MCMC), have been widely employed for this purpose. However, these methods often encounter challenges in high-dimensional spaces or distributions with intricate geometries. Moreover, MCMC methods, which construct a Markov chain with a stationary distribution aligned with the target distribution, often suffer from slow convergence due to long mixing times.

In light of the effectiveness of diffusion processes in generative modeling, there is significant interest in leveraging these processes for sampling.
The aim is to find diffusion processes such that starting with the samples from a tractable distribution such as Gaussian, the diffusion process should produce a sample from the desired distribution at the final time. Unlike conventional MCMC methods, diffusion-based approaches do not require tuning of proposal distributions or acceptance probabilities. Furthermore, diffusion-based methods seem to mitigate the slow convergence of MCMC methods.

Traditional score-based generative modeling generates samples from a target distribution by learning how to reverse a forward diffusion process that maps the target distribution to a prior distribution. Ideally, these diffusion processes require an infinite time horizon for convergence.
In this work, we design computationally efficient diffusion-based samplers using ideas from the \textit{stochastic interpolants} framework, and \textit{forward-backward SDEs}, to generate exact samples from the target distribution within a finite time.
\\[4pt]
\textbf{Problem Statement:}
We are interested in sampling from a probability distribution $\pi$ with support on $\R^d$ by constructing diffusion processes which run for a finite time. We are given an unnormalized probability density function $\hat{\pi}$, such that $\pi(x) = \frac{\hat{\pi}(x)}{\int\hat{\pi}(x)dx}$.

\textbf{Notations:}
We denote a Gaussian distribution (and its density) with mean $\mu$ and covariance $\Sigma$ by $\cN{\mu,\Sigma}$. A uniform random variable in the interval $[0,T]$ is denoted by $U([0,T])$. A derivative  of a scalar function $f$ with respect to time $t\in \mathbb{R}_+$ is denoted by $\dot{f}$. A gradient with respect to the space variables $x\in \mathbb{R}^d$ is denoted with $\nabla$, and $\Delta$ denotes the Laplacian. We use $\Vert\cdot\Vert$ for the  Euclidean norm in $\mathbb{R}^d$.

\subsection{Our Contributions}

There are infinitely many diffusion processes that can drive samples from a prior distribution to a target distribution in a finite time. However, computationally tractable ways to find and simulate diffusions that have distributions other than Dirac distribution at the initial time and given target distribution at finite time are still unknown to the best of our knowledge. We take a step towards this by providing a principled approach for generating such diffusions when the prior distribution is Gaussian. \emph{Specifically, we propose a class of diffusion-based sampling methods that, starting with samples from the Gaussian distribution, can produce samples from a target distribution in finite time given its unnormalized density.} We achieve this by taking an approach based on the stochastic interpolants framework \cite{albergo_stochastic_2023}. To the best of our knowledge, this is the first time that stochastic interpolants have been used for classical sampling. Our approach reduces the sampling problem to solving Hamilton-Jacobi-Bellman (HJB) equations; a class of well-studied partial differential equations (PDEs) that arise frequently in the field of optimal control. Traditionally, HJB PDEs are solved by minimizing the corresponding control costs \cite{zhang_path_2022}. Instead, for important reasons that will become clear later (see the discussion in Section~\ref{sec:solving_fbsde}), \emph{our approach uses the theory of forward-backward stochastic differential equations (FBSDE) that connects solutions of HJB PDEs (more generally, nonlinear parabolic PDEs) to solutions of a certain set of stochastic differential equations called FBSDEs}. 
Moreover, we solve these FBSDEs using machine learning-based methods \cite{han_solving_2018,e_algorithms_2022,raissi_forward-backward_2018}. One of the advantages of our methods is that they allow solving HJB PDEs without the need to compute computationally expensive Neural SDE gradients \cite{zhang_path_2022,berner_optimal_2023}. The techniques that we develop to solve HJB PDEs using FBSDEs can be of independent interest.  

\subsection{Our Techniques}
We draw inspiration from the stochastic interpolants framework, which begins with a family of time-indexed random variables defining the density of the diffusion process at each time instant $t$, and then explores methods to realize such a diffusion process. Such a collection of random variables that have desired probability distribution at the time boundaries are known as stochastic interpolants \cite{albergo_stochastic_2023}. 
To illustrate our techniques, in this section we restrict our focus to one of our methods based on \textit{half interpolants}. 
A half interpolant is a collection of random variables $\{x_t\}_{t\in[0,T]}$ given by $x_t = g(t)x^*+r(t)z,$ for $0\le t\le T$, where $g,r:[0,T]\rightarrow \R_+$ are functions such that $\frac{g}{r}$ is a non-decreasing. Furthermore, $g$ satisfies the boundary condition $g(0)=0$. Here, $x^*\sim\nu$ and $z\sim\cN{0,I_d}$ are independent random variables.

 Our aim is to implement a diffusion process $\{S_t\}_{t\in[0,T]}$ such that the distribution of $S_t$ is same as $x_t$ for all $t\in[0,T]$ and also enforce the constraint that $x_T,S_T\sim\pi$ (thereby implicitly choosing $\nu$). If realized, such a diffusion process can drive samples $S_0\sim\cN{0,r^2(0)I_d}$ to $S_T\sim\pi$. 
 
 Let $\rho(t,\cdot)$ denote the probability density of $x_t$ and let $s(t,x) := \nabla\log{\rho(t,x)}$ denote the so-called score function of density $\rho$. Lemma~\ref{lemma:pde_density} shows that $\rho$ satisfies a ``Fokker-Planck'' PDE given by
\begin{equation}\label{eqn:FPE_tech}
    \partial_t\rho-\frac{\eps^2(t)}{2}\Delta\rho+\nabla\cdot\left(\left(b(t,x)+\frac{\eps^2(t)}{2}s(t,x)\right)\rho\right) = 0,
\end{equation}
with initial condition $\rho(0,\cdot) \equiv \cN{0,r^2(0)I_d}$, where $b(t,x) = \dot g(t)\expectCond{x^*}{x_t=x}-r(t)\dot r(t)s(t,x)$ and $\eps:[0,T]\rightarrow \R_+$ is an arbitrary function.

Equation~\ref{eqn:FPE_tech} suggests that we can realize a process $S_t$ with density $\rho$ by simulating a stochastic differential equation (SDE) given by:
\begin{equation}\label{eqn:SDEforSampling_tech}
   dS_t = \left(b(t,S_t)+\frac{\eps^2(t)}{2}s(t,S_t)\right)dt+\eps(t)dW_t,
\end{equation}
with $S_0\sim\cN{0,r^2(0)I_d}$, where $\{W_t\}_{t\in{0,T}}$ is a standard Brownian motion.
Therefore, it suffices to have access to functions $b$ and $s$ (responsible for the drift term) to implement a process $S_t$ that has the same marginal distribution as $x_t$. We will show (Lemma~\ref{lemma:b_and_s}) that both $b$ and $s$ can be expressed in terms of $\expectCond{x^*}{x_t=x}$. Thus, it is sufficient to learn the function $\expectCond{x^*}{x_t=x}$. Towards this, for some $\beta:[0,T]\rightarrow\R_+$, we consider a function $u:[0,T]\times\R^d\rightarrow\R$ given by
\begin{align}\label{eqn:velocity_denf_tech}
    u(t,x) &= \log\frac{\rho(t,\beta(t)x)}{\psi(t,\beta(t)x)}\\ &= \log\int_{\R^d}\nu(x^*)e^{\frac{\beta(t)g(t)}{r^2(t)}<x,x^*>-\frac{g^2(t)}{2r^2(t)}\norm{x^*}^2}dx^*,\nonumber
\end{align}
where $\psi(t,\cdot)$ is the density of the isotropic Gaussian with variance $r^2(t)$. Taking the gradient of (\ref{eqn:velocity_denf_tech}), we note that $\expectCond{x^*}{x_t=x} = \frac{r^2(t)}{\beta(t)g(t)}\nabla u(t,\frac{x}{\beta(t)})$. Hence, a feasible way to obtain $\expectCond{x^*}{x_t=x}$ is to compute $\nabla u$. A direct calculation (Lemma~\ref{lemma:velociyt_HJB_PDE}) shows that $u$ satisfies the following HJB equation:
\begin{equation}\label{eqn:velocity_pde_tech}
    \partial_t u + \frac{\sigma^2}{2}\Delta u + \frac{\sigma^2}{2}\norm{\nabla u}^2-\partial_t\log\left(\frac{\beta(t)g(t)}{r^2(t)}\right)x^T\nabla u = 0,
\end{equation}
where $\sigma^2(t) = 2\frac{r^2(t)}{\beta^2(t)}\partial_t\log\frac{g(t)}{r(t)}$. The condition that $\frac{g}{r}$ is non-decreasing assures that $\sigma^2$ is a positive function. 
Observe that (\ref{eqn:velocity_pde_tech}) is a backward Kolmogorov PDE and can be solved given a terminal condition. To satisfy the constraint $x_T\sim\pi$, we want $\rho(T,\cdot)=\pi(\cdot)$, which gives the terminal condition $u(T, x)= \varphi(x) \equiv \log\frac{\pi(\beta(T)x)}{\psi(T,\beta(T)x)}$. The function $\beta$ is a design parameter, which we can choose such that the coefficients of the PDE are well-defined for $t\in[0,T]$. 

There are several ways to obtain the solution $u$ (more importantly $\nabla u$) of HJB PDE (\ref{eqn:velocity_pde_tech}) under the terminal condition $\varphi$. In the optimal control literature, a prominent approach for solving HJB PDEs involves minimizing the sum of control costs and the terminal cost (see Lemma~\ref{lemma_app:oc_optimization}). 
Instead, we exploit the connections between the solutions of non-linear PDEs and solutions to the corresponding FBSDE to solve (\ref{eqn:velocity_pde_tech}). The remainder of our approach then relies on machine learning-based techniques to solve an FBSDE associated with the PDE (\ref{eqn:velocity_pde_tech}). Solving the FBSDE gives us access to the function $\nabla u$ on an appropriate domain. Once we have the $\nabla u$, subsequently we obtain functions $b$ and $s$. We then can realize the process $S_t$ using (\ref{eqn:SDEforSampling_tech}). Figure~\ref{fig:trajectory_tech} shows sample trajectories of the diffusion process thus obtained when the target distribution is a mixture of Gaussians.

\begin{figure}
  \centering
  \includegraphics[scale = 0.3]{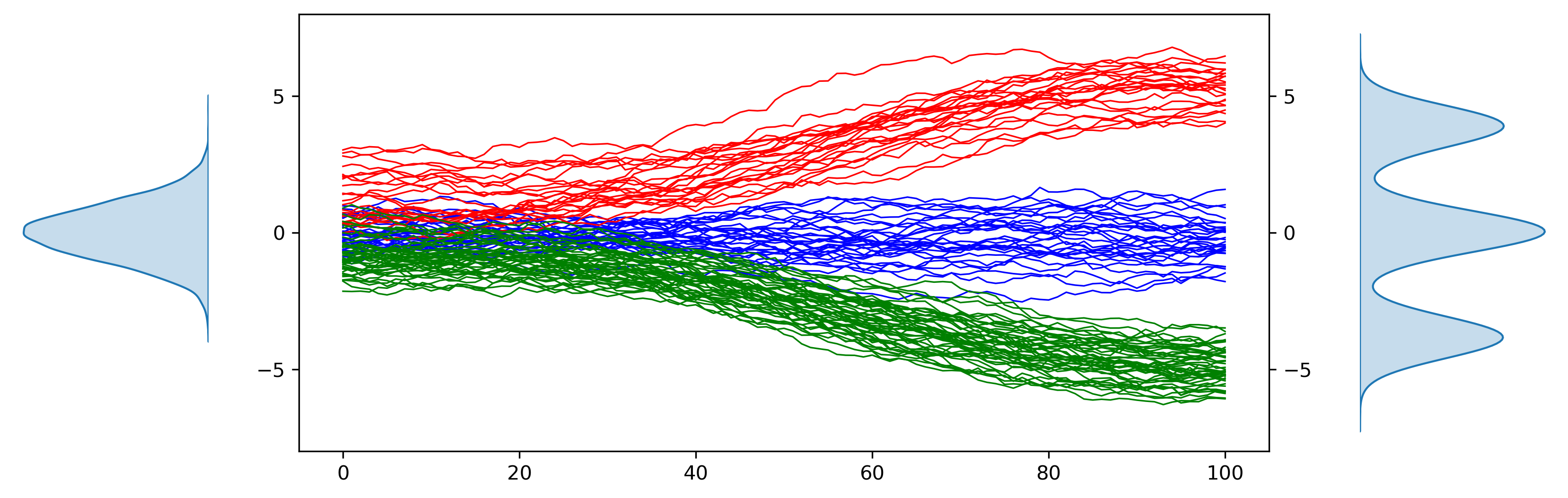}
  \caption{Sample trajectories of diffusion process for sampling from Gaussian mixture.}
  \label{fig:trajectory_tech}
\end{figure}

\subsection{Related Works}
\paragraph{MCMC methods:} For decades, MCMC has stood as the primary method for sampling from unnormalized densities. Techniques that integrate MCMC with annealing and importance sampling methods have consistently yielded superior results in this domain. Among these, Annealed Importance Sampling (AIS) \cite{neal_annealed_2001} and its Sequential Monte Carlo (SMC) \cite{del_moral_sequential_2006} extensions are widely regarded as state-of-the-art in numerous sampling tasks. Nonetheless, in many practical scenarios, the convergence of these methods can be notably slow. Additionally, analyzing their performance can pose significant challenges, further complicating their application in real-world settings.
\paragraph{Diffusion-based methods:} 
The utilization of diffusions for sampling has been prevalent for a considerable period, with the Langevin diffusion standing out as a prominent example. However, the usage of non-equilibrium dynamics of diffusions for sampling has gained popularity only recently. Noteworthy examples of diffusion-based samplers include Path Integral Sampler (PIS) \cite{zhang_path_2022}, Denoised Diffusion Sampler (DDS) \cite{vargas_denoising_2022}, and time reversed Diffusion Sampler (DIS) \cite{berner_optimal_2023}, Generalized Bridge Sampler (GBS) \cite{richter_improved_2023}, among others. These samplers leverage advancements in machine learning to address an optimization problem, with the solution being a control function that guides samples from a prior density to samples from the target density. For a detailed comparison of the performance of these methods, we refer the reader to \cite{blessing_beyond_2024}. More recently, the concept of time-reversing diffusion processes has been combined with MCMC techniques to develop samplers that do not require training \cite{grenioux_stochastic_2024,huang_reverse_2023}.
\paragraph{Stochastic interpolants:} 
The framework of stochastic interpolants was recently introduced in \cite{albergo_stochastic_2023}. Despite its conceptual simplicity, this framework provides a unified approach to utilizing diffusions for sampling, particularly in generative modeling tasks. Stochastic interpolants play a significant role in the derivation of our methods. In particular, our methods strive to learn certain quantities related to the densities defined by the stochastic interpolants. Subsequently, these quantities are used for sampling.
\paragraph{Schr\"odinger bridge:}
For arbitrary prior and target distributions, the task of finding a diffusion process that maps one to the other can be formulated as an optimization problem known as a the Schr\"odinger bridge problem. Concretely, the dynamical formulation of Schr\"odinger Bridge is the optimization problem $\min_{\mathbb{Q}\in\mathcal{P}(\mathbb{P}_0,\mathbb{P}_T)} D_{\text{KL}}(\mathbb{Q}||\mathbb{P})$, where $\mathcal{P}(\mathbb{P}_0,\mathbb{P}_T)$ is the set of all path measures having density $\mathbb{P}_0$ at $t=0$ and $\mathbb{P}_T$ at $t=T$ and $\mathbb{P}$ is a reference path measure. Solving a Schr\"odinger bridge problem is generally challenging, as it requires solving a set of coupled partial differential equations (PDEs) \cite{chen_likelihood_2021}. However, an instance of the Schrödinger bridge problem that is relatively easier to solve arises when the prior distribution $\mathbb{P}_0$ is a Dirac distribution. In this scenario, the Schrödinger bridge problem reduces to solving a single Hamilton-Jacobi-Bellman (HJB) PDE. The resultant diffusion process is known as a F\"ollmer process \cite{follmer_time_1986}. The PIS \cite{zhang_path_2022} algorithm--a special case of the sampling method we propose--is an implementation of F\"ollmer process.

\section{Preliminaries: Relation between diffusions and PDEs}
In this section we introduce the connections between diffusions generated by Stochastic Differential Equations (SDEs) and Partial Differential Equations (PDEs). We first recall the well known relation between 
diffusions and Fokker-Planck equations, and then we briefly review the connection between FBSDEs and non-linear parabolic PDEs. 

Let $\{W_t\}_{t\in[0,T]}$ be a standard Brownian motion. Consider a stochastic process $\{X_t\}_{t\in[0,T]}$ generated by the SDE:
\begin{equation*}
    dX_t = \mu(t,X_t)dt+\sigma(t)dW_t,\quad X_0\sim\nu,
\end{equation*}
where $\mu:[0,T]\times\R^d\rightarrow\R^d$ is the drift coefficient and $\sigma:[0,T]\rightarrow\R$ is the diffusion coefficient. The diffusion coefficient $\sigma$ can in general take a matrix form and may vary as a function of space. However, for the sake of simplicity in our presentation, we limit our discussion to the scenario where $\sigma$ is a scalar-valued function solely dependent on $t$. Let $\rho:[0,T]\times\R^d\rightarrow\R_+$ denote the probability density of $\{X_t\}_{t\in[0,T]}$, i.e., $X_t\sim\rho(t,\cdot)$. Then, $\rho$ satisfies a Fokker-Planck equation given by
\begin{equation*}
    \partial_t\rho-\frac{\sigma^2}{2}\Delta\rho+\nabla\cdot(\mu\rho)=0,\quad \rho(0,x) = \nu(x).
\end{equation*}
This constitutes a forward Kolmogorov PDE, which is well-defined as an initial value problem. Next, consider the following backward Kolmogorov PDE also known as \textit{quasi-linear parabolic partial differential equation}:
\begin{equation}\label{eqn:semiParaPDE}
    \partial_t u+ \frac{\sigma^2}{2}\Delta u + \mu^T\nabla u + f(t,x,u,\sigma^T\nabla u) = 0,
\end{equation}
with terminal condition $u(T,x) = \varphi(x)$, where $u:[0,T]\times \R^d\rightarrow\R$ is the solution of the PDE, $\mu:[0,T]\times\R^d\times\R\times\R^d\rightarrow\R^d$ and $\sigma:[0,T]\rightarrow\R$ are coefficient functions, $f:[0,T]\times\R^d\times\R\times\R^d\rightarrow\R$ is a non-linearity function, and $\varphi:\R^d\rightarrow\R$ a given terminal condition. The solution to PDE (\ref{eqn:semiParaPDE}) is related to diffusion processes via the so-called forward-backward stochastic differential equations (FBSDE). Consider the following set of stochastic differential equations:
\begin{align}\label{eqn:pardouxSDE}
\begin{split}
    dX_t &= \mu (t,X_t,Y_t,Z_t)dt + \sigma (t)dW_t,\\
    dY_t &= -f(t,X_t,Y_t,Z_t)dt + Z_t^TdW_t, \,\, Y_T = \varphi(X_T)
\end{split}
\end{align}
where $(X_t,Y_t,Z_t)$ are stochastic processes adapted to the natural filtration of $W_t$. Following pioneering work of Bismuth \cite{bismut_conjugate_1973}, it was shown by Pardoux and Peng \cite{pardoux_adapted_1990,pardoux_backward_1998,pardoux_forward-backward_1999} that under certain regularity conditions on $\mu,\sigma,$ and $f$, quite remarkably, there exists a unique solution $(X_t,Y_t,Z_t)$ to the above set of SDEs. It can be shown that the solution must satisfy $Y_t = u(t,X_t)$ and $Z_t = \sigma(t)\nabla u(t,X_t)$, where $u$ is the solution to PDE (\ref{eqn:semiParaPDE}). Thus we can solve the PDE (\ref{eqn:semiParaPDE}) by solving FBSDE (\ref{eqn:pardouxSDE}).

\section{Main Ideas}
In this section, we present our methods in detail. We introduce two algorithms for sampling from probability distributions--both based on the stochastic interpolants framework and the FBSDE formulation for solving partial differential equations. First, we review the definition of a linear interpolant (similar to the one-sided linear stochastic interpolant in \cite{albergo_stochastic_2023}).

\begin{definition}[Linear interpolants] For some $T>0$, let $g,r:[0,T]\rightarrow \R_+$ be such that $\frac{g}{r}$ is a non-deceasing function. Let $x^*\sim\nu$ and $z\sim\cN{0,I_d}$ be independent random variables. A \textit{Linear interpolant} is a collection of random variables $\{x_t\}_{t\in[0,T]}$ given by 
\begin{equation}\label{eqn:half_interpolant}
x_t = g(t)x^*+r(t)z,\quad 0\le t\le T. 
\end{equation}   
We call $x_t$ a \textit{half interpolant} if we include the boundary condition $g(0)=0$. Further, we call $x_t$ a \textit{full interpolant} if $g$ and $r$ satisfies the boundary conditions $g(0) = r(T)=0, g(T)=1$. We don't enforce any condition on $r(0)$.
\end{definition}
Figure~\ref{fig:linearInterpolants} in Appendix~\ref{appndx:interpolants} shows some examples of linear interpolants. Observe that, since $g/r$ is a non-decreasing function, $x_t$ becomes more informative about $x^*$ as time progresses. For a full interpolant, $x_T$ is fully informative about $x^*$ while for a half interpolant, $x_T$ is still a noisy version of $x^*$. We will demonstrate that we can utilize either of these to develop sampling methods. When using half interpolants for sampling, the distribution $\nu$ will be set implicitly by the constraint that we want $x_T\sim\pi$ whereas, for full interpolants, we take $\nu$ equal to the target density $\pi$. First, we describe a sampling method using half interpolants and later extend it to the case of full interpolants.
\subsection{Sampling using half  interpolants}
Consider a half interpolant $x_t$ for $t\in[0,T]$. Note that it is not possible to obtain $x_t$ using (\ref{eqn:half_interpolant}),  since it requires knowledge about $\nu$, and samples $x^*$ from $\nu$. Instead, if we can implement a diffusion process $\{S_t\}_{t\in[0,T]}$ such that the distribution of $S_t$ is same as $x_t$ for all $t\in[0,T]$ and also enforce the condition that $S_T\sim\pi$, then we have a method that drives $S_0\sim\cN{0,r^2(0)I_d}$ to $S_T\sim\pi$. Constructing such processes is our aim here.  Let $\rho(t,\cdot)$ denote the probability density of $x_t$
and $s(t,x) := \nabla\log{\rho(t,x)}$ denote the score funciton of the density $\rho$. The score function $s$ holds significant importance in generative modeling through diffusion processes. Notably, the score function is a pivotal component in our formulation as well. First, we present Lemma~\ref{lemma:pde_density} that characterizes the density function $\rho$ as a solution to certain PDEs \cite{albergo_stochastic_2023} (see Lemma~\ref{lemma_app:pde_density} in Appendix~\ref{sec:Proofs} for a proof).
\begin{lemma}\label{lemma:pde_density}
The probability density function of $x_t$ defined in (\ref{eqn:half_interpolant}) satisfies a PDE given by
\begin{equation}\label{eqn:first_order_pde}
    \partial_t\rho+\nabla\cdot(b\rho) = 0,\quad \rho(0,\cdot) = \cN{0,r^2(0)I_d},
\end{equation}
where $b(t,x) = \dot g(t)\expectCond{x^*}{x_t=x}-r(t)\dot r(t)s(t,x)$. Equivalently, $\rho$ satisfies the following Fokker-Planck equation:
\begin{equation}\label{eqn:FPE}
    \partial_t\rho-\frac{\eps^2(t)}{2}\Delta\rho+\nabla\cdot\left(\left(b(t,x)+\frac{\eps^2(t)}{2}s(t,x)\right)\rho\right) = 0,
\end{equation}
with initial condition $\rho(0,\cdot) = \cN{0,r^2(0)I_d}$.
\end{lemma}

Equation~\ref{eqn:first_order_pde} is the continuity equation for the density of particles in a velocity field $b$. This implies that a process $S_t$ with probability density $\rho(t,\cdot)$ can be obtained by solving an ODE with a random initial condition, given by
\begin{equation}\label{eqn:ODEforSampling}
dS_t = b(t,S_t)dt,\quad S_0\sim\cN{0,r^2(0)I_d}.    
\end{equation}
Similarly, equation~\ref{eqn:FPE} governs the evolution of the density of particles under a drift and diffusion component. That is, we can realize process $S_t$ by simulating the following SDE:
\begin{equation}\label{eqn:SDEforSampling}
   dS_t = \left(b(t,S_t)+\frac{\eps^2(t)}{2}s(t,S_t)\right)dt+\eps(t)dW_t,
\end{equation}
with $S_0\sim\cN{0,r^2(0)I_d}$, where $\{W_t\}_{t\in[0,T]}$ is a standard Brownian motion. Hence, it suffices to have access to functions $b$ and $s$ to implement a diffusion process $S_t$ that has the same marginal distribution as $x_t$. Moreover, Lemma~\ref{lemma:b_and_s} shows that having either $\expectCond{x^*}{x_t=x}$ or $s$ is sufficient to derive functions $b$ and $s$ (see Lemma~\ref{lemma_app:b_and_s} in Appendix~\ref{sec:Proofs} for a proof).
\begin{lemma}\label{lemma:b_and_s}
    Let $x_t$ be a linear or half interpolant. Let $s(t,x) = \nabla \rho(t,x)$ and $b(t,x) = \dot g(t)\expectCond{x^*}{x_t=x}-r(t)\dot r(t)s(t,x)$. Then, we have the following expressions for $b$ and $s$:
\begin{align}\label{eqn:b_and_s}
\begin{split}
  b(t,x) &= \begin{cases}
      \frac{\dot r(t)}{r(t)}x+\left(\dot g(t)-\frac{g(t)\dot r(t)}{r(t)}\right)\expectCond{x^*}{x_t=x}\\
      \frac{\dot g(t)}{g(t)}x+\left(r^2(t)\frac{\dot g(t)}{g(t)}-\dot r(t)r(t)\right)s(t,x)
      \end{cases}, \\
s(t,x) &= \frac{g(t)\expectCond{x^*}{x_t=x}-x}{r^2(t)}.
\end{split}  
\end{align}
\end{lemma}

Relying on the score $s$ to obtain $b$ presents a challenge due to the condition $g(0)=0$, causing the first term in the second expression for $b$ to diverge as $t$ approaches 0, resulting in numerical instability. Since there are no boundary conditions on $r$ for a half interpolant, a better strategy is to obtain the function $\expectCond{x^*}{x_t=x}$ and use the first and third expressions to estimate $b$ and $s$ respectively. Towards this, consider the expression for $\rho$:
\begin{align*}
    \rho(t,x) &= \frac{1}{(2\pi r^2(t))^{d/2}}\int_{\R^d}\nu(dx^*)e^{-\frac{\norm{x-g(t)x^*}^2}{2r^2(t)}}\\
    &= \psi(t,x)\int_{\R^d}\nu(dx^*)e^{\frac{g(t)}{r^2(t)}<x,x^*>-\frac{g^2(t)}{2r^2(t)}\norm{x^*}^2},
\end{align*}
where $\psi(t,x)=\frac{1}{(2\pi r^2(t))^{d/2}}e^{-\frac{\norm{x}^2}{2r^2(t)}}$. For some $\beta:[0,T]\rightarrow\R_+$, we define the function
\begin{align}\label{eqn:velocity}
\begin{split}
    u(t,x) &= \log\frac{\rho(t,\beta(t)x)}{\psi(t,\beta(t)x)}\\
    &= \log\int_{\R^d}\nu(dx^*)e^{\frac{\beta(t)g(t)}{r^2(t)}<x,x^*>-\frac{g^2(t)}{2r^2(t)}\norm{x^*}^2}.
\end{split}
\end{align}
Taking the gradient of $u$, we see that $\expectCond{x^*}{x_t=x} = \frac{r^2(t)}{\beta(t)g(t)}\nabla u(t,\frac{x}{\beta(t)})$.
Lemma~\ref{lemma:velociyt_HJB_PDE} characterizes $u$ as the solution to an HJB PDE (refer to Lemma~\ref{lemma_app_app:velociyt_HJB_PDE} in Appendix~\ref{sec:Proofs} for a proof).
\begin{lemma}\label{lemma:velociyt_HJB_PDE}
    Let $u:[0,T]\times\R^d:\rightarrow\R$ be the function given in (\ref{eqn:velocity}). Then $u$ satisfies the following Hamilton-Jacobi-Bellman equation
\begin{equation}\label{eqn:velocity_pde}
    \partial_t u + \frac{\sigma^2}{2}\Delta u + \frac{\sigma^2}{2}\norm{\nabla u}^2 + \mu^T\nabla u = 0,
\end{equation}
where $\mu(t,x) = -\partial_t\log\left(\frac{\beta(t)g(t)}{r^2(t)}\right)x$ and $\sigma^2(t) = 2\frac{r^2(t)}{\beta^2(t)}\left(\frac{\dot g(t)}{g(t)}-\frac{\dot r(t)}{r(t)}\right)\geq 0$. 
\end{lemma}
Note that (\ref{eqn:velocity_pde}) is a backward Kolmogorov PDE and can be solved given a terminal condition $u(T,\cdot)$. To constrain that $X_T\sim\pi$, we want $\rho(T,\cdot)=\pi(\cdot)$, which together with (\ref{eqn:velocity}) yields
\begin{equation}\label{eqn:terminal_cond_vel_pde}
 u(T, x)=\varphi(x) \equiv \log\frac{\pi(\beta(T)x)}{\psi(T,\beta(T)x)}.   
\end{equation}
Note that this condition also implicitly determines $\nu$. Here, for the sake of simplicity in presentation, we assumed that $\pi$ is a normalized density. However, it is important to note that it can be replaced by an unnormalized density $\hat{\pi}$. This would merely shift the solution $u$ by a constant, which would not impact the sampling algorithm, as only $\nabla u$ is required in the sampling phase. The function $\beta$ is a design parameter, which we can choose such that the coefficients of the PDE are well defined for $t\in[0,T]$. Note that simply choosing $\beta(t) = 1$ can cause $\mu$ and $\sigma$ to diverge as $t$ goes to $0$. One possible choice for $\beta$ is $\frac{r}{g}$, which makes the coefficients well defined for all $t\in[0,T]$. 

We emphasize that our aim is to obtain the function $\expectCond{x^*}{x_t=x} = \frac{r^2(t)}{\beta(t)g(t)}\nabla u(t,\frac{x}{\beta(t)})$ along the sample paths $S_t$ given by our sampling equation (\ref{eqn:SDEforSampling}). A prominent strategy in optimal-control literature to obtain $\nabla u$ is to formulate an optimization problem whose solution is $\nabla u$ (see discussion in Appendix~\ref{sec:oc_based_approach}). 
However, for reasons explained below, we pursue an alternative approach to solve PDE (\ref{eqn:velocity_pde}). Given that (\ref{eqn:velocity_pde}) is a non-linear parabolic PDE, as previously established, we can derive its solutions by solving the corresponding FBSDE. An FBSDE corresponding to PDE (\ref{eqn:velocity_pde}) can be formulated as:
\begin{align}\label{eqn:fbsde_half_interpolant}
\begin{split}   
    dX_t &= \left(\mu(t,X_t)+\sigma(t) Z_t\right)dt+\sigma(t)dW_t,\\
    dY_t &= \frac{1}{2}\norm{Z_t}^2dt + Z_t^TdW_t,
\end{split}
\end{align}
with $X_0=\xi$ an appropriately chosen initial condition and $Y_T=\varphi(X_T)$. Lemma~\ref{lemma:pde_fbsde} relates the processes in (\ref{eqn:fbsde_half_interpolant}) to the solution of the PDE (\ref{eqn:velocity_pde}) under the terminal condition (\ref{eqn:terminal_cond_vel_pde}) (refer to Lemma~\ref{lemma_app:pde_fbsde} in Appendix~\ref{sec:Proofs} for a proof).
\begin{lemma}\label{lemma:pde_fbsde}
    Let $u:[0,T]\times\R^d:\rightarrow\R$ be the solution to PDE (\ref{eqn:velocity_pde}). Then the processes $Y_t$ and $Z_t$ in (\ref{eqn:fbsde_half_interpolant}) are given by $Y_t=u(t,X_t)$ and $Z_t=\sigma(t)\nabla u(t,X_t)$.
\end{lemma}
\subsubsection{Solving the FBSDE}\label{sec:solving_fbsde}
We take a machine learning-based approach to solve the FBSDE (\ref{eqn:fbsde_half_interpolant}). In particular, we approximate the solution $u$ to the PDE (\ref{eqn:velocity_pde}) with a neural network $u^\theta$ paramterized by $\theta$. Subsequently, we obtain an approximation to $\nabla u$ by taking the gradient of $u^\theta$. These approximations enable us to approximate the processes $Y_t$ and $Z_t$. Thereafter, we can form a loss function based on the fact that $(X_t,Y_t,Z_t)$ satisfies the SDEs given in (\ref{eqn:fbsde_half_interpolant}). A subtle point in solving the above FBSDE (\ref{eqn:fbsde_half_interpolant}) is that the distribution of the process $X_t$ differs from that of the sampling process $S_t$ at each time instant $t$. Therefore, we might not obtain the function $\nabla u$ on the desired domain. We overcome this difficulty by noting that FBSDE (\ref{eqn:fbsde_half_interpolant}) is a local equation in the sense that it should be satisfied at each time instant $t$. In particular, we maintain a separate process $X_t$ in order to ensure that we learn $u$ and $\nabla u$ on the appropriate domain, and use a penalty term in the loss (Equ. \eqref{lossFBSDElocal}) which enforces constraints arising from FBSDE locally in time. This is a major advantage of FBSDE over optimal control based approaches for solving HJB PDEs.  Specifically, for a $\tau$ chosen uniformly at randomly from $[0,T]$ we set, $X=X_\tau$, 
$Y=u^\theta(\tau,X_\tau)$, $Z=\sigma(\tau)\nabla u^\theta(\tau,X_\tau)$, where $X_\tau$ is the solution to the following ODE (for $t=\tau)$:
\begin{equation}\label{eqn:fbsde_init_process}
    \frac{dX_t}{dt} = \frac{\dot r(t)}{r(t)}X_t+\left(\frac{\dot g(t)}{g(t)}-\frac{\dot r(t)}{r(t)}\right)\frac{r^2(t)}{\beta(t)}\nabla u^\theta\left(t,\frac{X_t}{\beta(t)}\right),
\end{equation}
with $X_0\sim\cN{0,r^2(0)I_d}$ and we 
form the following $\delta$-step discretization of the FBSDE as:
\begin{align}\label{eqn:fbsde_discretized_half_int}
\begin{split}
        \hat{X}_\delta &= X + \left(\mu(\tau,X)+\sigma(\tau)Z\right)\delta+\sigma(\tau)\sqrt{\delta}w,\\
       \hat{Y}_\delta &= Y + \frac{1}{2}\norm{Z}^2\delta+\sqrt{\delta}Z^Tw,
\end{split}
\end{align}
where $w\sim\cN{0,I_d}$. 
The ODE in (\ref{eqn:fbsde_init_process}) coincides with the ODE in (\ref{eqn:ODEforSampling}) when $\nabla u^\theta$ equals $\nabla u$. Therefore, at optimum, we ensure that $\nabla u$ is learned over an appropriate domain. With an appropriately choosen $\lambda>0$, we use the equations (\ref{eqn:fbsde_discretized_half_int}), the quantity $Y_\delta = u^\theta(\tau+\delta,\hat{X}_\delta)$, and the boundary conditions of the FBSDE (\ref{eqn:fbsde_half_interpolant}), to form the loss function as follows:
\begin{equation}\label{lossFBSDElocal}
\cL(\theta) = \mathbb{E}{ \norm{\nabla u^\theta(T,X_{T})-\nabla\varphi(X_T)}^2}+\lambda\mathbb{E}{\left(\hat{Y}_{\delta}-Y_{\delta}\right)^2}. 
\end{equation}
The expectations above are over the process $X_\tau$ (given $\tau$), the random time $\tau$, and $w$. Note that since we are more interested in the quantity $\nabla u$ than $u$, we formed the loss function in terms of the gradient of the terminal condition.  We find the approximate solution to FBSDE (\ref{eqn:fbsde_half_interpolant}) and in turn to the PDE (\ref{eqn:velocity_pde}) by minimizing $\cL(\theta)$ over $\theta$. In practice, we discretize the ODE (\ref{eqn:fbsde_init_process}) using the Euler-Maruyama scheme. We observe that the discretization error in this step does not directly contribute to the training error, as it only affects the precise path of $X_t$. We present the steps to compute loss in Procedure~\ref{proc:loss_half_int} in Appendix \ref{apndx:half_interpolant}.

\subsubsection{Sampling}
Once we have an estimate of $\nabla u$, we can derive the functions $b$ and $s$ by utilizing equations (\ref{eqn:b_and_s}). Substituting $\expectCond{x^*}{x_t=x}=\frac{r^2(t)}{\beta(t)g(t)}\nabla u(t,\frac{x}{\beta(t)})$, we get
\begin{align*}
    b(t,x) &= \frac{\dot r(t)}{r(t)}x+\frac{r^2(t)}{\beta(t)}\left(\frac{\dot g(t)}{g(t)}-\frac{\dot r(t)}{r(t)}\right)\nabla u\left(t,\frac{x}{\beta(t)}\right),\\
    s(t,x) &= \frac{1}{\beta(t)}\nabla u\left(t,\frac{x}{\beta(t)}\right) -\frac{x}{r^2(t)}.
\end{align*}
Subsequently, we can use either ODE (\ref{eqn:ODEforSampling}) or SDE (\ref{eqn:SDEforSampling}) for sampling. We outline the steps for sampling in Procedure~\ref{proc:sample_half_int} in Appendix \ref{apndx:half_interpolant}.

\subsection{Sampling using full interpolants}\label{sec:full_interpolant}
In this section, we introduce a method to sample from probability distributions using a full interpolant. Merely employing the previous method is destined to fail, as the coefficients of the PDE tend towards infinity if we impose the condition $r(T)=0$ required by the full interpolant.  Instead, we propose a two-step approach that involves solving two HJB PDEs; one for $t\in[0,T']$ and the other for $t\in[T',T]$ with $T'\in(0,T)$. The idea is to obtain $\expectCond{x^*}{x_t=x}$ for $t\in[0,T']$ and the score $s$ for $t\in[T',T]$. This is due to the observation that it is possible to form well-behaved PDEs for a quantity related to $\expectCond{x^*}{x_t=x}$ for $t\in[0,T']$ and similarly, for a quantity related to $s$ for $t\in[T',T]$. Since $b$ can be derived from either $\expectCond{x^*}{x_t=x}$ or $s$, we acquire the functions essential for sampling for all $t$ within the whole interval $[0,T]$.

First, let us consider $t$ in the interval $[T',T]$. For some function $\alpha:[T',T]\rightarrow R_+$, let $v(t,x) = \log\rho(t,\alpha(t)x)+d\log g(t)$. Taking the gradient of $v$, we observe that $v$ is related to the score by $s(t,x )=\frac{1}{\alpha(t)}\nabla v(t,\frac{x}{\alpha(t)})$. Lemma~\ref{lemma:score_pde} shows that $v$ satisfies an HJB PDE (see Lemma~\ref{lemma_app:score_pde} in Appendix~\ref{sec:Proofs} for a proof).
\begin{lemma}\label{lemma:score_pde}
    Let $v(t,x) = \log\rho(t,\alpha(t)x)+d\log g(t)$. Then $v$ satisfies the following Hamilton-Jacobi-Bellman equation
\begin{equation}\label{eqn:score_pde}
    \partial_t v+\frac{\bar{\sigma}^2}{2}\Delta v+\frac{\bar{\sigma}^2}{2}\norm{\nabla v}^2+\bar{\mu}^T\nabla v=0,
\end{equation}
where $\bar{\mu}(t,x) = \partial_t\log\left({\frac{g(t)}{\alpha(t)}}\right)x$ and $\bar{\sigma}^2 = 2\frac{r^2(t)}{\alpha^2(t)}\left(\frac{\dot g(t)}{g(t)}-\frac{\dot r(t)}{r(t)}\right)$.
\end{lemma}

For sampling, we need $x_T$ distributed as $\pi$. This provides us with the terminal condition
\begin{equation}\label{eqn:score_term_condition}
 v(T,x) = \varphi(x)\equiv \log\pi(\alpha(T)x)+d\log g(T).   
\end{equation}
 Similar to $\beta$ in the case of half interpolants, we can select $\alpha$ such that the PDE (\ref{eqn:score_pde}) is well defined for all $t$ in $[T',T]$. Here, simply choosing $\alpha = 1$ also suffices.

 For $t$ in $[0,T']$, it is natural to employ the function $u$ used in the half interpolant-based sampler. Therefore, we have to solve the PDE (\ref{eqn:velocity_pde}), this time however, subjected to a terminal condition dictated by the solution to PDE (\ref{eqn:score_pde}). The terminal condition is given by $\varphi'(x)= \log\frac{\rho(T',\beta(T')x)}{\psi(T',\beta(T')x)} =  v\left(T',\frac{\beta(T')x}{\alpha(T')}\right)-\log\psi(T',\beta(T')x)-d\log g(T').$ For completeness, we restate the PDE below:
\begin{equation}\label{eqn:pde_velocity_interpolant}
    \partial_t u + \frac{\sigma^2}{2}\Delta u + \frac{\sigma^2}{2}\norm{\nabla u}^2 + \mu^T\nabla u = 0,\quad u(T',x) = \varphi'(x).
\end{equation}
To estimate $u$ and $v$, we need to solve PDEs (\ref{eqn:pde_velocity_interpolant}) and (\ref{eqn:score_pde}) subjected to the terminal conditions $\varphi'$ and $\varphi$ respectively. We again rely on solving the associated FBSDEs for solving the PDEs. The details of FBSDEs and the loss function for solving them can be found in Appendix~\ref{apndx:full_interpolant}.

\section{Numerical Results}
We evaluated our methods on several distributions from the literature, that are considered challenging to sample from. More details of our implementation can be found in Section~\ref{apndx:implemnetation_details}. Below, we provide some distributions for which we present our results. Let $\Psi(x;\mu,\Sigma)$ denote the probability density of the $d$-dimensional Gaussian distribution with mean $\mu$ and covariance $\Sigma$. 
\begin{figure*}[ht]
  \centering
  \includegraphics[scale = 0.38]{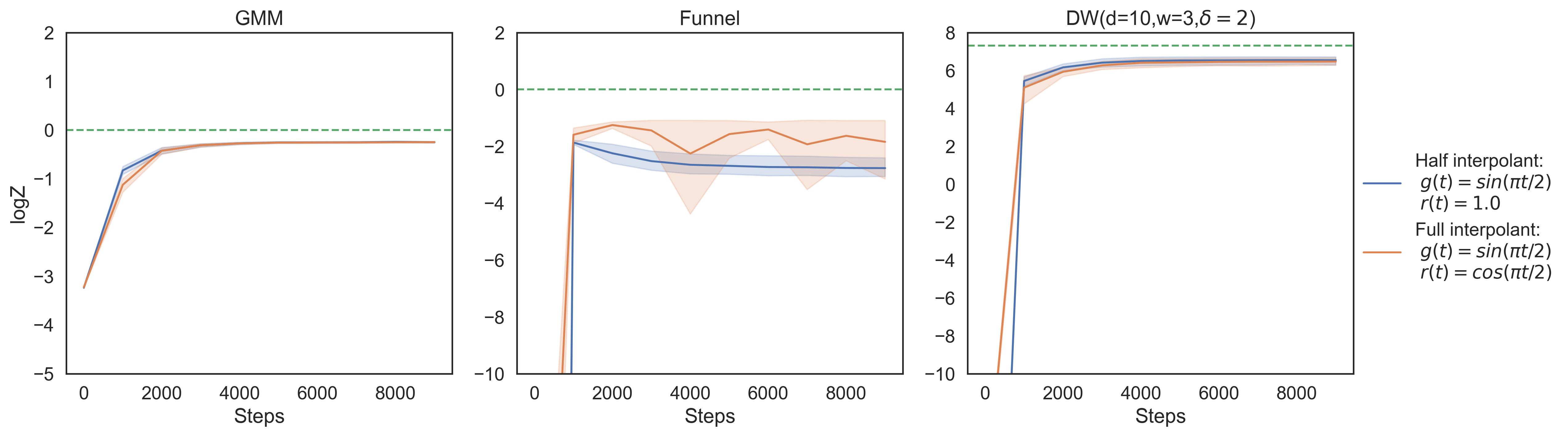}
  \caption{Estimates of $\log Z$ as a function of training steps along with $95\%$ confidence intervals.}
  \label{fig: results}
\end{figure*}

\paragraph{Gaussian mixture model (GMM):} The density of a Gaussian mixture model is given by $\pi(x) = \frac{1}{M}\sum_{i=1}^{M} \Psi(x;\mu_i,\Sigma_i)$. In our experiments, we use $M=9, \{\mu_i\}_{i=1}^9 = \{-5,0,5\}^2$ and $\Sigma_i = 0.3I_2$ for all $i$ (here $d=2$).
\paragraph{Neal's Funnel distribution:} Probability density of the funnel distribution is given by $\pi(x) = \Psi(x_1;0,s^2)\Psi(x_2, x_3,\dots, x_d;0,e^{x_1}I_{d-1})$. The values of the parameters are $d=10, s=3$. Reference \cite{neal_slice_2003} introduces this distribution and discusses the behavior of MCMC implementations. 
\paragraph{Double Well (DW):} The probability density of double well is given by $\pi(x) = \frac{1}{Z}\exp\left(-\sum_{i=1}^w(x_i^2-\delta)^2-\frac{1}{2}\sum_{i=w+1}^d x_i^2\right)$.

\paragraph{Soft Spherical Spin Glass model :} We consider a spin glass model whose probability density is given by $\pi(x) \propto \exp\{\frac{\beta}{\sqrt{2d}}x^TAx-\frac{\beta^2}{4d}\norm{x}^4-\frac{1}{2}\norm{x}^2\}$, where $A\in\R^{d\times d}$ with i.i.d. Gaussian entries.

We use our sampler for estimating different quantities associated to the above distributions. We estimate the \textit{log normalization constant} ($\log Z$) of the distribution using the method explained in Appendix~\ref{apndx:est_logZ}. Other quantities we consider are mean $L_1$-norm, and mean squared $L_2$-norm, whose estimates are obtained from samples using empirical average. We stress that our estimates are computed without using the importance sampling technique. Hence they reflect the `true' quality of the samples. In practice, we observe that the full interpolant-based sampling method gives better results compared to the half interpolant-based sampler. We mention that the full interpolant-based sampling method that we presented can in fact work with half interpolant functions as well. Figure~\ref{fig: results} shows the estimates of log normalization constant as a function of the training steps, given by the full interpolant-based sampling method. The interpolants used are visualized in Figure~\ref{fig:linearInterpolants}.

Additional numerical results can be found in Appendix~\ref{apndx:additionalExps}. In Section~\ref{sec:gaussian_spin_glass}, we use our sampler to confirm a predicted phase transition in the soft spherical spin glass model. In Section~\ref{sec:comparison_gbs} we compare our method against another diffusion based sampler called Generalized Bridge Sampler \cite{richter_improved_2023,blessing_beyond_2024}. In Table~\ref{tab:ablation} we show how different parameters affect the performance of our sampler.

\section{Discussion and Future Work}\label{sec:discussion}
We presented a class of diffusion-based algorithms for sampling from unnormalized densities that can obtain samples in finite time. An existing diffusion-based method that can achieve the same is the Path Integral Sampler (PIS) \cite{zhang_path_2022}. However, PIS is obliged to have a Dirac distribution as its prior. On the other hand, diffusion-based samplers such as DDS \cite{vargas_denoising_2022} and DIS \cite{berner_optimal_2023} are more closely related to score-based generative modeling techniques. They can accommodate non-Dirac distributions as their prior, but ideally require an infinite amount of time for convergence. Our method circumvents these limitations.  Our approach which is based on the stochastic interpolants framework permits Gaussian initial distribution on top of producing samples in finite time. We would like to point out that the F\"ollmer process that the PIS implements is same as the diffusion process that our half interpolant-based method realizes when the interpolant functions take the form $g(t)= f(t)$ and $r(t) = \sqrt{f(t)}$ for some positive function $f$ with $f(0)=0$.

The training phase of PIS, DDS, and DIS involves computing gradients over the paths generated by an SDE (Neural SDE), which can become computationally expensive as we decrease the discretization step-size of the SDE. However, our method does not suffer from this limitation as we detach the process (\ref{eqn:fbsde_init_process}) from the computational graph, and also the discretization step-size of the ODE \ref{eqn:fbsde_init_process} does not significantly impact the final accuracy. An extension of DIS presented in \cite{richter_improved_2023} also detaches path gradients from the computational graph, making it computationally less expensive. This is achieved by using log-variance divergence as the loss function.

A limitation of our method currently is that it is not evident if our construction allows for the computation of important weights required to correct the estimates obtained using the samples. This is an important aspect for further investigation. Lastly, we mention that the performance of our sampler can vary significantly based on the values of the hyperparameters used in the implementation. Our primary focus in this work is to present ideas that offer a class of more versatile diffusion-based sampling algorithms, while also offering a fresh perspective on the existing ones.

\section*{Acknowledgements}
The work of A. J. G has been supported by Swiss National Science Foundation grant number 200021-204119. 

\bibliography{references.bib} 

\begin{thebibliography}{}

\bibitem[Albergo et~al., 2023]{albergo_stochastic_2023}
Albergo, M.~S., Boffi, N.~M., and Vanden-Eijnden, E. (2023).
\newblock Stochastic {Interpolants}: {A} {Unifying} {Framework} for {Flows} and {Diffusions}.
\newblock arXiv:2303.08797 [cond-mat].

\bibitem[Arous et~al., 2001]{arous_aging_2001}
Arous, G.~B., Dembo, A., and Guionnet, A. (2001).
\newblock Aging of spherical spin glasses.
\newblock {\em Probability Theory and Related Fields}, 120(1):1--67.

\bibitem[Barra et~al., 2014]{barra_about_2014}
Barra, A., Genovese, G., Guerra, F., and Tantari, D. (2014).
\newblock About a solvable mean field model of a {Gaussian} spin glass.
\newblock {\em Journal of Physics A: Mathematical and Theoretical}, 47(15):155002.

\bibitem[Berner et~al., 2023]{berner_optimal_2023}
Berner, J., Richter, L., and Ullrich, K. (2023).
\newblock An optimal control perspective on diffusion-based generative modeling.
\newblock {\em Transactions on Machine Learning Research}.

\bibitem[Bismut, 1973]{bismut_conjugate_1973}
Bismut, J.-M. (1973).
\newblock Conjugate convex functions in optimal stochastic control.
\newblock {\em Journal of Mathematical Analysis and Applications}, 44(2):384--404.

\bibitem[Blessing et~al., 2024]{blessing_beyond_2024}
Blessing, D., Jia, X., Esslinger, J., Vargas, F., and Neumann, G. (2024).
\newblock Beyond {ELBOs}: {A} {Large}-{Scale} {Evaluation} of {Variational} {Methods} for {Sampling}.
\newblock arXiv:2406.07423 [cs, stat].

\bibitem[Chen et~al., 2021]{chen_likelihood_2021}
Chen, T., Liu, G.-H., and Theodorou, E. (2021).
\newblock Likelihood {Training} of {Schrödinger} {Bridge} using {Forward}-{Backward} {SDEs} {Theory}.
\newblock In {\em International {Conference} on {Learning} {Representations}}.

\bibitem[Del~Moral et~al., 2006]{del_moral_sequential_2006}
Del~Moral, P., Doucet, A., and Jasra, A. (2006).
\newblock Sequential {Monte} {Carlo} samplers.
\newblock {\em Journal of the Royal Statistical Society: Series B (Statistical Methodology)}, 68(3):411--436.
\newblock \_eprint: https://onlinelibrary.wiley.com/doi/pdf/10.1111/j.1467-9868.2006.00553.x.

\bibitem[E et~al., 2022]{e_algorithms_2022}
E, W., Han, J., and Jentzen, A. (2022).
\newblock Algorithms for {Solving} {High} {Dimensional} {PDEs}: {From} {Nonlinear} {Monte} {Carlo} to {Machine} {Learning}.
\newblock {\em Nonlinearity}, 35(1):278--310.

\bibitem[Föllmer, 1986]{follmer_time_1986}
Föllmer, H. (1986).
\newblock Time reversal on {Wiener} space.
\newblock In Albeverio, S.~A., Blanchard, P., and Streit, L., editors, {\em Stochastic {Processes} — {Mathematics} and {Physics}}, pages 119--129, Berlin, Heidelberg. Springer.

\bibitem[Grenioux et~al., 2024]{grenioux_stochastic_2024}
Grenioux, L., Noble, M., Gabrié, M., and Durmus, A.~O. (2024).
\newblock Stochastic {Localization} via {Iterative} {Posterior} {Sampling}.

\bibitem[Han et~al., 2018]{han_solving_2018}
Han, J., Jentzen, A., and E, W. (2018).
\newblock Solving high-dimensional partial differential equations using deep learning.
\newblock {\em Proceedings of the National Academy of Sciences}, 115(34):8505--8510.
\newblock arXiv:1707.02568.

\bibitem[Hendrycks and Gimpel, 2023]{hendrycks_gaussian_2023}
Hendrycks, D. and Gimpel, K. (2023).
\newblock Gaussian {Error} {Linear} {Units} ({GELUs}).
\newblock arXiv:1606.08415 [cs].

\bibitem[Ho et~al., 2020]{ho_denoising_2020}
Ho, J., Jain, A., and Abbeel, P. (2020).
\newblock Denoising {Diffusion} {Probabilistic} {Models}.
\newblock In {\em Advances in {Neural} {Information} {Processing} {Systems}}, volume~33, pages 6840--6851. Curran Associates, Inc.

\bibitem[Huang et~al., 2023]{huang_reverse_2023}
Huang, X., Dong, H., Hao, Y., Ma, Y.-A., and Zhang, T. (2023).
\newblock Reverse {Diffusion} {Monte} {Carlo}.

\bibitem[Kingma and Ba, 2017]{kingma_adam_2017}
Kingma, D.~P. and Ba, J. (2017).
\newblock Adam: {A} {Method} for {Stochastic} {Optimization}.
\newblock arXiv:1412.6980 [cs].

\bibitem[Kosterlitz et~al., 1976]{kosterlitz_spherical_1976}
Kosterlitz, J.~M., Thouless, D.~J., and Jones, R.~C. (1976).
\newblock Spherical {Model} of a {Spin}-{Glass}.
\newblock {\em Physical Review Letters}, 36(20):1217--1220.
\newblock Publisher: American Physical Society.

\bibitem[Neal, 2001]{neal_annealed_2001}
Neal, R.~M. (2001).
\newblock Annealed importance sampling.
\newblock {\em Statistics and Computing}, 11(2):125--139.

\bibitem[Neal, 2003]{neal_slice_2003}
Neal, R.~M. (2003).
\newblock Slice sampling.
\newblock {\em The Annals of Statistics}, 31(3):705--767.
\newblock Publisher: Institute of Mathematical Statistics.

\bibitem[Pardoux, 1998]{pardoux_backward_1998}
Pardoux, E. (1998).
\newblock Backward {Stochastic} {Differential} {Equations} and {Viscosity} {Solutions} of {Systems} of {Semilinear} {Parabolic} and {Elliptic} {PDEs} of {Second} {Order}.
\newblock In Decreusefond, L., Øksendal, B., Gjerde, J., and Üstünel, A.~S., editors, {\em Stochastic {Analysis} and {Related} {Topics} {VI}}, Progress in {Probability}, pages 79--127, Boston, MA. Birkhäuser.

\bibitem[Pardoux and Peng, 1990]{pardoux_adapted_1990}
Pardoux, E. and Peng, S.~G. (1990).
\newblock Adapted solution of a backward stochastic differential equation.
\newblock {\em Systems \& Control Letters}, 14(1):55--61.

\bibitem[Pardoux and Tang, 1999]{pardoux_forward-backward_1999}
Pardoux, E. and Tang, S. (1999).
\newblock Forward-backward stochastic differential equations and quasilinear parabolic {PDEs}.
\newblock {\em Probability Theory and Related Fields}, 114(2):123--150.

\bibitem[Raissi, 2018]{raissi_forward-backward_2018}
Raissi, M. (2018).
\newblock Forward-{Backward} {Stochastic} {Neural} {Networks}: {Deep} {Learning} of {High}-dimensional {Partial} {Differential} {Equations}.
\newblock arXiv:1804.07010 [cs, math, stat].

\bibitem[Richter et~al., 2023]{richter_improved_2023}
Richter, L., Berner, J., and Liu, G.-H. (2023).
\newblock Improved sampling via learned diffusions.
\newblock arXiv:2307.01198 [cs, math, stat].

\bibitem[Song and Ermon, 2019]{song_generative_2019}
Song, Y. and Ermon, S. (2019).
\newblock Generative {Modeling} by {Estimating} {Gradients} of the {Data} {Distribution}.
\newblock In {\em Advances in {Neural} {Information} {Processing} {Systems}}, volume~32. Curran Associates, Inc.

\bibitem[Song et~al., 2021]{song_score-based_2021}
Song, Y., Sohl-Dickstein, J., Kingma, D.~P., Kumar, A., Ermon, S., and Poole, B. (2021).
\newblock Score-{Based} {Generative} {Modeling} through {Stochastic} {Differential} {Equations}.
\newblock arXiv:2011.13456 [cs, stat].

\bibitem[Tancik et~al., 2020]{tancik_fourier_2020}
Tancik, M., Srinivasan, P., Mildenhall, B., Fridovich-Keil, S., Raghavan, N., Singhal, U., Ramamoorthi, R., Barron, J., and Ng, R. (2020).
\newblock Fourier {Features} {Let} {Networks} {Learn} {High} {Frequency} {Functions} in {Low} {Dimensional} {Domains}.
\newblock In {\em Advances in {Neural} {Information} {Processing} {Systems}}, volume~33, pages 7537--7547. Curran Associates, Inc.

\bibitem[Vargas et~al., 2022]{vargas_denoising_2022}
Vargas, F., Grathwohl, W.~S., and Doucet, A. (2022).
\newblock Denoising {Diffusion} {Samplers}.
\newblock In {\em International {Conference} on {Learning} {Representations}}.

\bibitem[Zhang and Chen, 2022]{zhang_path_2022}
Zhang, Q. and Chen, Y. (2022).
\newblock Path {Integral} {Sampler}: a stochastic control approach for sampling.
\newblock arXiv:2111.15141 [cs].

\end{thebibliography}
\newpage

\appendix

\onecolumn
\aistatstitle{Supplementary Materials}
\section{Algorithms for Half Interpolant Sampler}\label{apndx:half_interpolant}
We compile the steps required to form the loss for half interpolant sampler into an algorithm. First, we define a sub-routine~\ref{proc:loss_fbsde_half_int} that generates the component in the loss due to FBSDE equations. This subroutine will also be used in the full interpolant sampler. 
\begin{procedure}
\DontPrintSemicolon
	\KwIn{$t,X,m,\delta,f,h$}
	\KwOut{L}
    $w\sim\cN{0,I_d}$\;
    $Z \gets h(t)\nabla m(t,X)$\;
    $\hat{Y} \gets m(t,X)+\frac{1}{2}\norm{Z}^2\delta+\sqrt{\delta}{Z}^Tw$\; 
    $X \gets X + \left(f(t,X)+h(t)Z\right)\delta+h(t)\sqrt{\delta}w$\;
    $Y \gets m(t+\delta,X)$\;
    $L \gets \left(\hat{Y}-Y\right)^2$\;
\Return{L}
\caption{LossFBSDE($t,X,m,\delta,f,h$)}\label{proc:loss_fbsde_half_int}
\end{procedure} 
Next, we present the procedure~\ref{proc:loss_half_int} to generate the overall loss for the half interpolant sampler. Here, we assume that the batch size equals 1. In practice, we use multiple realizations of $X$ and $\{t_i\}_{i=1}^{N-1}$ to form a single batch. This loss is subsequently minimized to find the optimal $\theta$.
\begin{procedure}
\DontPrintSemicolon
	\KwIn{$\theta,T,N,\delta,\lambda$}
	\KwOut{$\cL$}
$t_0 \gets 0, t_N \gets T$\;
$(t_1\le t_2\le\cdots\le t_{N-1}) \gets U([0,T])^{N-1}$\; 
$X\sim\cN{0,r^2(0)I_d}$\;
$\cL \gets 0$\;
\For{$i = 0,\cdots,N-1$}{
    $\cL \gets \cL + \lambda*\text{LossFBSDE}(t_i,X,u^\theta,\delta,\mu,\sigma)$\;
    $X \gets X + \left(\frac{\dot r(t_i)}{r(t_i)}X+\left(\dot g(t_i)-\frac{g(t_i)\dot r(t_i)}{r(t_i)}\right)\frac{r^2(t_i)}{\beta(t_i)g(t_i)}\nabla u^\theta\left(t_i,\frac{X}{\beta(t_i)}\right)\right)(t_{i+1}-t_i)$\;
}
$\cL\gets \cL + \norm{\nabla u^\theta(T,X)-\nabla\varphi(X)}^2$\;
\Return{$\cL$}	\caption{LossHalfInterpolant()}\label{proc:loss_half_int}
\end{procedure}
Once we have found the optimal $\theta$, we can sample from the target distribution by simulating the SDE given in (\ref{eqn:SDEforSampling}). Procedure~\ref{proc:sample_half_int} describes the steps involved in sampling.
\begin{procedure}
\DontPrintSemicolon
	\KwIn{$\theta^*,T,N',\eps$}
	\KwOut{$S$}
$\Delta \gets \frac{T}{N'}$\;
$S\sim\cN{0,r^2(0)I_d}$\;
\For{$i = 0,\cdots,N'-1$}{
    $t = i*\Delta$\;
    $w\sim\cN{0,I_d}$\;
    $b \gets \frac{\dot r(t)}{r(t)}S+\frac{r^2(t)}{\beta(t)}\left(\frac{\dot g(t)}{g(t)}-\frac{\dot r(t)}{r(t)}\right)\nabla u^{\theta^*}\left(t,\frac{S}{\beta(t)}\right)$\;
    $s \gets \frac{1}{\beta(t)}\nabla u^{\theta^*}\left(t,\frac{S}{\beta(t)}\right) -\frac{S}{r^2(t)}$\;
    $S \gets S + \left(b+\frac{\eps^2(t)}{2}s\right)\Delta+\sqrt{\Delta}\eps(t)w$\;
}
\Return{$S$}	\caption{SampleHalfInterpolant()}\label{proc:sample_half_int}
\end{procedure} 

\section{Details of the Full Interpolant-based Sampler}\label{apndx:full_interpolant}
In Section~\ref{sec:full_interpolant}, we have seen that in a full interpolant-based sampler, we have to solve two HJB PDEs given by (\ref{eqn:pde_velocity_interpolant}) and (\ref{eqn:score_pde}) under respective terminal conditions. Similar to the half interpolant-based sampler, we solve the PDEs by solving associated FBSDEs. An FBSDE corresponding to PDE (\ref{eqn:pde_velocity_interpolant}) can be written as
\begin{equation}\label{eqn:fbsde_velocity_interpolant}
\begin{split}
    dX_t &= \left(\mu(t,X_t)+\sigma(t) Z_t\right)dt+\sigma(t)dW_t,\quad X_0 = \xi\\
    dY_t &= \frac{1}{2}\norm{Z_t}^2dt + Z_t^TdW_t,\quad Y_{T'} = \varphi'(X_{T'}),
\end{split}
\end{equation}
while for PDE (\ref{eqn:score_pde}), an FBSDE is
\begin{equation}\label{eqn:fbsde_score_interpolant}
\begin{split}   
    dX_t &= \left(\bar{\mu}(t,X_t)+\bar{\sigma}(t) Z_t\right)dt+\bar{\sigma}(t)dW_t,\quad X_{T'} = \xi'\\
    dY_t &= \frac{1}{2}\norm{Z_t}^2dt + Z_t^TdW_t,\quad Y_{T} = \varphi(X_T).
\end{split}
\end{equation}

\subsection{Solving the FBSDEs}
As in the case of half interpolants, we take a machine learning-based approach to solve FBSDEs (\ref{eqn:fbsde_score_interpolant}, \ref{eqn:fbsde_velocity_interpolant}). We approximate the solution to the PDEs (\ref{eqn:pde_velocity_interpolant}, \ref{eqn:score_pde}) $u,v$ with neural networks $u^\theta$ and $v^{\theta'}$ paramterized by $\theta$ and $\theta'$ respectively. Subsequently, we obtain approximations to processes $(Y_t, Z_t)$, which along with the FBSDEs (\ref{eqn:fbsde_score_interpolant},\ref{eqn:fbsde_velocity_interpolant}) help us form a loss function. Specifically, for $\tau,\tau'$ chosen independently and uniformly at randomly from $[0,T']$ and $[T',T]$ respectively,  we form the $\delta$-step discretization of the FBSDE (\ref{eqn:fbsde_velocity_interpolant}) as 
\begin{align*}
       X &= X_\tau,\quad Z = \sigma(\tau)\nabla u^\theta(\tau,X),\\
       \hat{Y}_\delta &= u^\theta(\tau,X)+\frac{1}{2}\norm{Z}^2\delta+\sqrt{\delta}Z^Tw,\\
    \hat{X}_\delta &= X + \left(\mu(\tau,X)+\sigma(\tau)Z\right)\delta+\sigma(\tau)\sqrt{\delta}w\\
    Y_\delta &= u^\theta(\tau+\delta,\hat{X}_\delta),
\end{align*}
and that of FBSDE (\ref{eqn:fbsde_score_interpolant}) as
\begin{align*}
       X' &= X_{\tau'}\quad Z' = \bar{\sigma}(\tau')\nabla v^{\theta'}(\tau',X')\\
       \hat{Y}'_{\delta} &= v^{\theta'}(\tau',X')+\frac{1}{2}\norm{Z'}^2\delta+\sqrt{\delta}Z'^Tw'\\
        \hat{X}'_{\delta} &= X'+ \left(\bar{\mu}(\tau',X')+\bar{\sigma}(\tau')Z'\right)\delta+\bar{\sigma}(\tau')\sqrt{\delta}w'\\
    Y'_{\delta} &= v^{\theta'}(\tau'+\delta,\hat{X}'_{\delta})\\
\end{align*}
where  $w,w'\overset{\text{i.i.d.}}{\sim}\cN{0,I_d}$ and $X_t$ is the solution to the following ODE:
\begin{align*}
    X_0 &\sim\cN{0,r^2(0)I_d}\\
    \frac{dX_t}{dt} &= \begin{cases}
        \frac{\dot r(t)}{r(t)}X_t+\left(\dot g(t)-\frac{g(t)\dot r(t)}{r(t)}\right)\frac{r^2(t)}{\beta(t)g(t)}\nabla u^\theta\left(t,\frac{X_t}{\beta(t)}\right),\quad 0\le t\le T'\\
        \frac{\dot g(t)}{g(t)}X_t+\frac{r(t)}{\alpha(t)}\left(\frac{\dot g(t)}{g(t)}r(t)-\dot r(t)\right)\nabla v^{\theta'}(t,\frac{X_t}{\alpha(t)}),\quad T'<t\le T.
    \end{cases}
\end{align*}
We then form a loss function as follows:
\begin{multline}
\cL(\theta,\theta') = \expect{ \norm{\nabla u^\theta(T',X_{T'})-\nabla\varphi'(X_{T'})}^2}+\lambda\expect{\left(\hat{Y}_{\delta}-Y_{\delta}\right)^2}\\
+\expect{ \norm{\nabla v^{\theta'}(T,X_{T})-\nabla\varphi(X_{T})}^2}+\lambda\expect{\left(\hat{Y}'_{\delta}-Y'_{\delta}\right)^2}.
\end{multline}
We find a solution to FBSDEs (\ref{eqn:fbsde_score_interpolant},\ref{eqn:fbsde_velocity_interpolant}) and in turn to the PDEs (\ref{eqn:pde_velocity_interpolant}, \ref{eqn:score_pde}) by minimizing $\cL(\theta,\theta')$ over $\theta, \theta'$. We compile the steps involved in computing the loss in Procedure~\ref{proc:loss_int}.
\begin{procedure}
\DontPrintSemicolon
	\KwIn{$\theta,T,T',N,\delta,\lambda$}
	\KwOut{$\cL$}
 $X\sim\cN{0,r^2(0)I_d}$\;
$\cL \gets 0$\;
$N' = \left\lceil \frac{NT'}{T}\right\rceil$\;
$t_0 \gets 0, t_{N'} \gets T'$\;
$(t_1\le t_2\le\cdots\le t_{N'-1}) \gets U([0,T'])^{N'-1}$\;
\For{$i = 0,\cdots,N'-1$}{
    $\cL \gets \cL + \lambda*\text{LossFBSDE}(t_i,X,u^\theta,\delta,\mu,\sigma)$\;
    $X \gets X + \left(\frac{\dot r(t_i)}{r(t_i)}X+\left(\dot g(t_i)-\frac{g(t_i)\dot r(t_i)}{r(t_i)}\right)\frac{r^2(t_i)}{\beta(t_i)g(t_i)}\nabla u^\theta\left(t_i,\frac{X}{\beta(t_i)}\right)\right)(t_{i+1}-t_i)$\;
}
$\cL\gets \cL + \norm{\nabla u^\theta(T',X)-\nabla\varphi'(X)}^2$\;
$t_0 \gets T', t_{N-N'} \gets T$\;
$(t_1\le t_2\le\cdots\le t_{N-N'-1}) \gets U([T',T])^{N-N'-1}$\;
\For{$i = 0,\cdots,N-N'-1$}{
    $\cL \gets \cL + \lambda*\text{LossFBSDE}(t_i,X,v^{\theta'},\delta,\bar{\mu},\bar{\sigma})$\;
    $X \gets X + \left(\frac{\dot g(t_i)}{g(t_i)}X+\frac{r(t_i)}{\alpha(t_i)}\left(\frac{\dot g(t_i)}{g(t_i)}r(t_i)-\dot r(t_i)\right)\nabla v^{\theta'}(t_i,\frac{X}{\alpha(t_i)})\right)(t_{i+1}-t_i)$\;
}
$\cL\gets \cL + \norm{\nabla v^{\theta'}(T,X)-\nabla\varphi(X)}^2$\;
\Return{$\cL$}
\caption{LossFullInterpolant()}\label{proc:loss_int}
\end{procedure}

\subsubsection{Sampling}
Once we have an estimate for $\nabla u, \nabla v$, we can estimate the functions $b$ and $s$ by using equations (\ref{eqn:b_and_s}). We have
\begin{equation*}
    s(t,x) = \begin{cases}
        \frac{1}{\beta(t)}\nabla u(t,\frac{x}{\beta(t)})-\frac{x}{r^2(t)},\quad &0\le t<T'\\
        \frac{1}{\alpha(t)}\nabla v(t,\frac{x}{\alpha(t)}),\quad &T'\le t\le T',
    \end{cases}
\end{equation*}
\begin{equation*}
    b(t,x) = \begin{cases}
        \frac{\dot r(t)}{r(t)}x+\frac{r^2(t)}{\beta(t)g(t)}\left(\dot g(t)-\frac{g(t)\dot r(t)}{r(t)}\right)\nabla u(t,\frac{x}{\beta(t)}),\quad &0\le t<T'\\
        \frac{\dot g(t)}{g(t)}x+\frac{r(t)}{\alpha(t)}\left(r(t)\frac{\dot g(t)}{g(t)}-\dot r(t)\right)\nabla v(t,\frac{x}{\alpha(t)}),\quad &T'\le t\le T'.
    \end{cases}
\end{equation*} 
Once we have estimates of $b$ and $s$, we can sample from $\pi$ by simulating ODE (\ref{eqn:ODEforSampling}) or SDE (\ref{eqn:SDEforSampling}). In the Procedure~\ref{proc:smaple_int}, we present the steps involved in the sampling phase.

\begin{procedure}
\DontPrintSemicolon
	\KwIn{$\theta,\theta',T,T',N',\eps$}
	\KwOut{$S$}
$\Delta \gets \frac{T}{N}$\;
$S\sim\cN{0,r^2(0)I_d}$\;
\For{$i = 0,\cdots,N-1$}{
    
    $t = i*\Delta$\;
    $w\sim\cN{0,I_d}$\;
    \eIf{$t\le T'$}{
    $b \gets \frac{\dot r(t)}{r(t)}S+\frac{r^2(t)}{\beta(t)}\left(\frac{\dot g(t)}{g(t)}-\frac{\dot r(t)}{r(t)}\right)\nabla u^\theta\left(t,\frac{S}{\beta(t)}\right)$\;
    $s \gets \frac{1}{\beta(t)}\nabla u^\theta\left(t,\frac{S}{\beta(t)}\right) -\frac{S}{r^2(t)}$\;
    }
    {
    $b \gets \frac{\dot g(t)}{g(t)}S+\frac{r(t)}{\alpha(t)}\left(\frac{\dot g(t)}{g(t)}r(t)-\dot r(t)\right)\nabla v^{\theta'}(t,\frac{S}{\alpha(t)})$\;
    $s \gets \frac{1}{\alpha(t)}\nabla v^{\theta'}(t,\frac{S}{\alpha(t)})$\;
    }
    $S \gets S + \left(b+\frac{\eps^2(t)}{2}s\right)\Delta+\sqrt{\Delta}\eps(t)w$
}
	\caption{SampleFullInterpolant()}\label{proc:smaple_int}
\end{procedure}

\section{Optimal Control-based Approach}\label{sec:oc_based_approach}
We remark that $\nabla u$ could a priori be obtained as a solution to an optimization problem (see Lemma~\ref{lemma_app:oc_optimization}). In particular, let $u$ be the solution to PDE (\ref{eqn:velocity_pde}) and for an $m\in C^1([0,T]\times\R^d,\R^d)$ let the process $X_t$ be generated by the SDE $dX_t = \left(\sigma^2 m(t,X_t)+\mu(t,X_t)\right)dt+\sigma(t) dW_t$. Then, we have 
\begin{equation}\label{eqn:oc_optimization}
    \nabla u = \argmin{m} \expectCond{\frac{1}{2}\int_0^T\sigma^2(t)\norm{m(s,X_s)}^2ds-\varphi(X_T)}{X_0}.
\end{equation}
Typically, $m$ is parameterized by a neural network, and its parameters are learned by solving the minimization problem in (\ref{eqn:oc_optimization}). However, observe that the domain over which the function $\nabla u$ is learned depends on the distribution of the diffusion process $X_t$ for the optimal $m$. In our case, this presents issues since the diffusion process employed for sampling (\ref{eqn:SDEforSampling}) is different from $X_t$. Hence, if solved using the optimal control approach, $\nabla u$ would be learned over a domain that differs from what is required.

\begin{lemma}\label{lemma_app:oc_optimization}
    Let $u$ be the solution to PDE (\ref{eqn:velocity_pde}) and let the process $X_t$ be generated by the SDE
    \begin{equation}\label{eqn_app:oc_diffusion}
         dX_t = \left(\sigma^2 m(t,X_t)+\mu(t,X_t)\right)dt+\sigma(t) dW_t.
    \end{equation}
    Then, we have 
\begin{align*}
    u(0,X_0) &= - \min_m \expectCond{\frac{1}{2}\int_0^T\sigma^2(t)\norm{m(s,X_s)}^2ds-\varphi(X_T)}{X_0},\\
    \nabla u &= \argmin{m} \expectCond{\frac{1}{2}\int_0^T\sigma^2(t)\norm{m(s,X_s)}^2ds-\varphi(X_T)}{X_0}.
\end{align*}
\end{lemma}
\begin{proof}
This result is a special example of the general theory of stochastic optimal control, but we present a short proof for the convenience of the reader.
    Applying Ito's lemma to $u(t,X_t)$ gives
\begin{align*}
    du(t,X_t) &= \partial_t u(t,X_t)dt+\nabla u(t,X_t)^TdX_t+\frac{\sigma^2}{2}\Delta u(t,X_t)dt,\\
    & = \left(\partial_t u(t,X_t)+\sigma^2 m^T\nabla u(t,X_t)+\mu^T\nabla u(t,X_t)+\frac{\sigma^2}{2}\Delta u(t,X_t)\right)dt+\sigma\nabla u^TdW_t,\\
    &= \left(\sigma^2 m^T\nabla u(t,X_t)-\frac{\sigma^2}{2}\norm{\nabla u(t,X_t)}^2\right)dt+\sigma\nabla u^TdW_t,\\
    &= -\frac{\sigma^2}{2}\left(\norm{\nabla u(t,X_t)-m}^2 -\norm{m}^2\right)dt+\sigma\nabla u^TdW_t.
\end{align*}
Integrating from $t$ to $T$ and taking expectation conditioned on $X_t$ gives
\begin{align*}
    u(t,X_t) &= \expectCond{\frac{\sigma^2}{2}\int_t^T\norm{\nabla u(s,X_s)-m(s,X_s)}^2ds}{X_t}\\
    &\qquad\qquad\qquad\qquad+ \expectCond{u(T,X_T) - \frac{\sigma^2}{2}\int_t^T\norm{m(s,X_s)}^2ds}{X_t},\\
    &\ge \expectCond{u(T,X_T) - \frac{\sigma^2}{2}\int_t^T\norm{m(s,X_s)}^2ds}{X_t}.\\
\end{align*}
and the equality is attained when $m(t,X_t)=\nabla u(t,X_t)$ a.s.
Thus, we get the following variational formulation:
\begin{align*}
    u(t,X_t) &= \max_m \expectCond{\varphi(X_T) - \frac{\sigma^2}{2}\int_t^T\norm{m(s,X_s)}^2ds}{X_t},\\
    &= - \min_m \expectCond{\frac{\sigma^2}{2}\int_t^T\norm{m(s,X_s)}^2ds-\varphi(X_T)}{X_t},\\
    \nabla u(t,X_t) &= \argmin{m} \expectCond{\frac{\sigma^2}{2}\int_t^T\norm{m(s,X_s)}^2ds-\varphi(X_T)}{X_t}.
\end{align*}
\end{proof}

\section{Interpolants}\label{appndx:interpolants}
Figure \ref{fig:linearInterpolants} shows some examples of linear interpolants that we used in our experiments. The first row has half interpolants and the second row contains full interpolants. We observed that keeping $\frac{\dot r(t)}{r(t)}$ close to $0$ for small values of $t$ is better for training. This is because, at the beginning of training, $\frac{\dot r(t)}{r(t)}X_t$ is the drift in (\ref{eqn:fbsde_init_process}). Having this negative will decrease the initial exploration of the space, while having it positive might cause stability issues. 
\begin{figure}
  \centering
  \includegraphics[scale = 0.6]{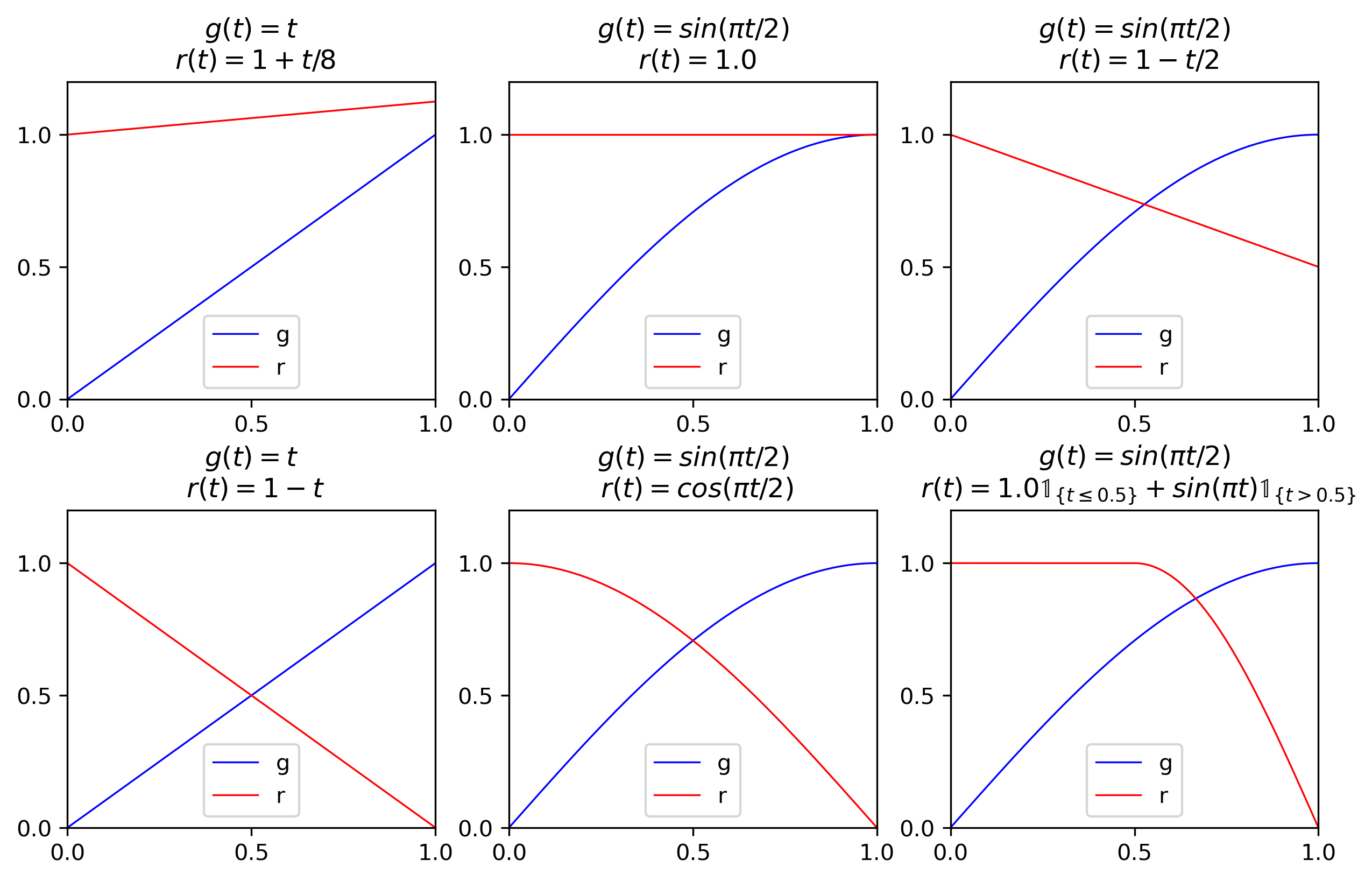}
  \caption{Examples of linear interpolants. Top row: half interpolants. Bottom row: full interpolants.}
  \label{fig:linearInterpolants}
\end{figure}

\section{Proofs of Lemmas}\label{sec:Proofs}
\begin{lemma}\label{lemma_app:pde_density}
The probability density function of $x_t$ defined in (\ref{eqn:half_interpolant}) satisfies a PDE given by
\begin{equation}\label{eqn_app:first_order_pde}
    \partial_t\rho+\nabla\cdot(b\rho) = 0,\quad \rho(0,\cdot) \equiv \cN{0,r^2(0)I_d},\quad \rho(T,\cdot) = \pi(\cdot),
\end{equation}
where $b(t,x) =  \dot g(t)\expectCond{x^*}{x_t=x}-r(t)\dot r(t)s(t,x),$. Equivalently, $\rho$ satisfies the following Fokker-Planck equation:
\begin{equation}\label{eqn_app:FPE}
    \partial_t\rho-\frac{\eps^2(t)}{2}\Delta\rho+\nabla\cdot\left(\left(b+\frac{\eps^2(t)}{2}s(t,x)\right)\rho\right) = 0,\quad \rho(0,\cdot) \equiv \cN{0,r^2(0)I_d}.
\end{equation}
\end{lemma}
\begin{proof}
    Let $\hat{\rho}(t,\cdot)$ be the characteristic funciton of $\rho(t,\cdot)$. That is,
    \begin{align*}
        \hat{\rho}(t,k) &= \int_{\R^d}\rho(t,x)e^{i\langle k,x\rangle}dx,\\
        &= \expect{e^{i\langle k,x_t\rangle}}.
    \end{align*}
    Taking derivative of $\hat{\rho}$ w. r. t. $t$, we get
    \begin{align*}
        \partial_t\hat{\rho} &= \left\langle ik,\expect{\dot x_t e^{i\langle k,x_t\rangle}}\right\rangle,\\
        &= \left\langle ik,\expect{\expectCond{\dot x_t}{x_t} e^{i\langle k,x_t\rangle}}\right\rangle.
    \end{align*}
    Taking the Fourier transform of above equation, we get
    \begin{align*}
        \partial_t\rho = \nabla\cdot(\expectCond{\dot x_t}{x_t=x}\rho).
    \end{align*}
    It remains to show that $b(t,x)=\expectCond{\dot x_t}{x_t}$. We have
    \begin{align*}
        \expectCond{\dot x_t}{x_t} &= \dot g(t)\expectCond{x^*}{x_t=x} + \dot r(t)\expectCond{z}{x_t=x},\\
        & = \dot g(t)\expectCond{x^*}{x_t=x}-r(t)\dot r(t)s(t,x),
    \end{align*}
    where we used the result that $\expectCond{z}{x_t=x} = -r(t)s(t,x)$.
    Fokker-Planck equation (\ref{eqn_app:FPE}) can be readily obtained by noting that $\nabla\cdot(s\rho) = \Delta\rho$.
\end{proof}

\begin{lemma}\label{lemma_app:b_and_s}
    Let $x_t$ be a linear stochastic interpolant. Let $s(t,x) = \nabla \rho(t,x)$ and $b(t,x) = \dot g(t)\expectCond{x^*}{x_t=x}-r(t)\dot r(t)s(t,x)$. Then, we have
\begin{equation}\label{eqn_app:b_and_s}
\begin{split}
  b(t,x) &= \begin{cases}
      \frac{\dot r(t)}{r(t)}x+\left(\dot g(t)-\frac{g(t)\dot r(t)}{r(t)}\right)\expectCond{x^*}{x_t=x},\\
      \frac{\dot g(t)}{g(t)}x+\left(r^2(t)\frac{\dot g(t)}{g(t)}-\dot r(t)r(t)\right)s(t,x),
      \end{cases}\\
s(t,x) &= \frac{g(t)\expectCond{x^*}{x_t=x}-x}{r^2(t)}.  
\end{split}
\end{equation}
\end{lemma}
\begin{proof}
    We have $x_t=g(t)x^*+r(t)z$. Taking expectation conditioned on $x_t=x$, we get
    \begin{equation*}
        x = g(t)\expectCond{x^*}{x_t=x}+r(t)\expectCond{z}{x_t=x}.
    \end{equation*}
    A direct computation (also known as the Tweedie's formula) yields that $\expectCond{z}{x_t=x} = -r(t)s(t,x)$. Hence we can derive $s(t,x)$ from $\expectCond{x^*}{x_t=x}$ and vice-versa. This proves the claims in the Lemma.
    
\end{proof}

\begin{lemma}\label{lemma_app_app:velociyt_HJB_PDE}
    Let $u:[0,T]\times\R^d:\rightarrow\R$ be the function given in (\ref{eqn:velocity}). Then $u$ satisfies the following Hamilton-Jacobi-Bellman equation
\begin{equation}\label{eqn_app:velocity_pde}
    \partial_t u + \frac{\sigma^2}{2}\Delta u + \frac{\sigma^2}{2}\norm{\nabla u}^2+\mu^T\nabla u = 0,
\end{equation}
where $\sigma^2(t) = 2\frac{r^2(t)}{\beta^2(t)}\left(\frac{\dot g(t)}{g(t)}-\frac{\dot r(t)}{r(t)}\right)$ and $\mu(t)= - \partial_t\log\left(\frac{\beta(t)g(t)}{r^2(t)}\right)x$.
\end{lemma}
\begin{proof}
    From the definition of $u$ in (\ref{eqn:velocity}), we can directly compute $\partial_tu, \nabla u$ and $\Delta u$, which gives
    \begin{align*}
        \partial_t u &= \frac{\int_{\R^d}\nu(dx^*)\left(\partial_t\frac{\beta(t)g(t)}{r^2(t)}<x,x^*>-\partial_t\frac{g^2(t)}{2r^2(t)}\norm{x^*}^2\right)e^{\frac{\beta(t)g(t)}{r^2(t)}<x,x^*>-\frac{g^2(t)}{2r^2(t)}\norm{x^*}^2}}{\int_{\R^d}\nu(dx^*)e^{\frac{\beta(t)g(t)}{r^2(t)}<x,x^*>-\frac{g^2(t)}{2r^2(t)}\norm{x^*}^2}}\\
        \nabla u &= \frac{\int_{\R^d}\nu(dx^*)\left(\frac{\beta(t)g(t)}{r^2(t)}\right)x^*e^{\frac{\beta(t)g(t)}{r^2(t)}<x,x^*>-\frac{g^2(t)}{2r^2(t)}\norm{x^*}^2}}{\int_{\R^d}\nu(dx^*)e^{\frac{\beta(t)g(t)}{r^2(t)}<x,x^*>-\frac{g^2(t)}{2r^2(t)}\norm{x^*}^2}}\\
        \nabla^2 u &= \frac{\int_{\R^d}\nu(dx^*)\left(\frac{\beta(t)g(t)}{r^2(t)}\right)^2x^*{x^*}^Te^{\frac{\beta(t)g(t)}{r^2(t)}<x,x^*>-\frac{g^2(t)}{2r^2(t)}\norm{x^*}^2}}{\int_{\R^d}\nu(dx^*)e^{\frac{\beta(t)g(t)}{r^2(t)}<x,x^*>-\frac{g^2(t)}{2r^2(t)}\norm{x^*}^2}}-\nabla u\nabla u^T\\
        \Delta u &= \frac{\int_{\R^d}\nu(dx^*)\left(\frac{\beta(t)g(t)}{r^2(t)}\right)^2\norm{x^*}^2e^{\frac{\beta(t)g(t)}{r^2(t)}<x,x^*>-\frac{g^2(t)}{2r^2(t)}\norm{x^*}^2}}{\int_{\R^d}\nu(dx^*)e^{\frac{\beta(t)g(t)}{r^2(t)}<x,x^*>-\frac{g^2(t)}{2r^2(t)}\norm{x^*}^2}}-\norm{\nabla u}^2.
    \end{align*}
This shows that $u$ satisfies the PDE (\ref{eqn_app:velocity_pde}).
\end{proof}

\begin{lemma}\label{lemma_app:pde_fbsde}
    Let $u:[0,T]\times\R^d:\rightarrow\R$ be the solution to PDE (\ref{eqn:velocity_pde}). Then the processes $Y_t$ and $Z_t$ in (\ref{eqn:fbsde_half_interpolant}) is given by $Y_t=u(t,X_t)$ and $Z_t=\sigma(t)\nabla u(t,X_t)$.
\end{lemma}
\begin{proof}
    We can rewrite the PDE (\ref{eqn:velocity_pde}) as 
    \begin{equation*}
        \partial_t u + \frac{\sigma^2}{2}\Delta u - \frac{\sigma^2}{2}\norm{\nabla u}^2 + (\mu+\sigma\nabla u)^T\nabla u = 0.
    \end{equation*}
    Thus, appealing to the connection between PDEs and FBSDEs established in the Preliminaries, we identify that $Y_t = u(t,X_t)$ and $Z_t=\sigma(t)\nabla u(t,X_t)$.
\end{proof}

\begin{lemma}\label{lemma_app:score_pde}
    Let $v(t,x) = \log\rho(t,\alpha(t)x)+d\log g(t)$. Then $v$ satisfies the following Hamilton-Jacobi-Bellman equation
\begin{equation}\label{eqn_app:score_pde}
    \partial_t v+\frac{\bar{\sigma}^2}{2}\Delta v+\frac{\bar{\sigma}^2}{2}\norm{\nabla v}^2+\partial_t\log{\frac{g(t)}{\alpha(t)}}x^T\nabla v=0,
\end{equation}
where $\bar{\sigma}^2 = 2\frac{r^2(t)}{\alpha^2(t)}\left(\frac{\dot g(t)}{g(t)}-\frac{\dot r(t)}{r(t)}\right)$.
\end{lemma}
\begin{proof}
\begin{align*}
    \partial_t v(t,x) &= (\partial_t\log\rho)(t,\alpha(t)x)+\dot\alpha x^T(\nabla\log\rho)(t,\alpha(t)x)+d\frac{\dot g(t)}{g(t)},\\
    &= -(b^T\nabla \log\rho)(t,\alpha x)-(\nabla\cdot b)(t,\alpha x)+\dot\alpha x^T(\nabla\log\rho)(t,\alpha(t)x)+d\frac{\dot g(t)}{g(t)},\\
    &= -\left(\frac{\dot g}{g}x+\frac{\bar{\sigma}^2}{2}\nabla v\right)^T\nabla v - \frac{\bar{\sigma}^2}{2}\Delta v + \frac{\dot \alpha}{\alpha}x^T\nabla v.
\end{align*}
where $\bar{\sigma}^2 = 2\frac{r^2(t)}{\alpha^2(t)}\left(\frac{\dot g(t)}{g(t)}-\frac{\dot r(t)}{r(t)}\right)$. Simplifying, we get
\begin{equation}
    \partial_t v+\frac{\bar{\sigma}^2}{2}\Delta v+\left(\left(\frac{\dot g}{g}-\frac{\dot \alpha}{\alpha}\right)x+\frac{\bar{\sigma}^2}{2}\nabla v\right)^T\nabla v=0.
\end{equation}
\end{proof}

\section{Implementation Details}\label{apndx:implemnetation_details}
We provide the details of our implementation here. Our implementation is based on the JAX\footnote{https://jax.readthedocs.io/en/latest/index.html} framework. We rely on the automatic differentiation feature in JAX to compute the gradients of functions required in our algorithms. Table~\ref{tab:hyperparams} gives the hyperparameters of the sampler that we used in our simulations. 
 All our simulations are performed with $T=1$. We note that choosing a small $T$ would require the neural network to represent functions with a large Lipshitz constant, thus increasing the complexity of the neural network. We next present the details of the neural networks used to represent the solution to the PDEs (\ref{eqn:velocity_pde}),(\ref{eqn:score_pde}) and (\ref{eqn:pde_velocity_interpolant}). 

\subsection{Neural networks}
In score-based generative modeling, the score function is parameterized directly by a neural network. Our formulation of solving PDEs using FBSDE requires us to have a representation for log-likelihood. Consequently, we utilize energy-based models and obtain the score function by taking the gradient of log-likelihood. However, since we consider unnormalized densities, we use only the score function to obtain the loss. Similar to the approach taken in \cite{zhang_path_2022,vargas_denoising_2022,berner_optimal_2023}, we integrate the terminal condition of PDEs into the architecture of the neural network. In particular, we use a neural network of the following form:
\begin{equation*}
     \Phi^\theta(t,x) = \Phi_1^{\theta_1}(t,x)+\Phi_2^{\theta_2}(t)\phi(x),
\end{equation*}
where $\theta = \{\theta_1,\theta_2\}$ and $\phi$ is the terminal condition of the PDE under consideration. We initialize the networks $\Phi_1^{\theta_1}$ with $0$ and $\Phi_2^{\theta_2}$ with $1$ so that $u^\theta(T,\cdot)=\phi(\cdot)$ in the beginning. Both networks use Fourier Embedding \cite{tancik_fourier_2020} to encode time. The architecture of the neural networks $\Phi_1^{\theta_1}$ and $\Phi_2^{\theta_2}$ can be described as follows:
$$
\Phi_1^{\theta_1}(t,x) = L_1\circ\varrho\circ L\circ\varrho\circ L\circ\varrho\bigg(L(x)+L\circ\varrho\circ L\circ\text{Fourier}(t)\bigg), 
$$
$$
\Phi_2^{\theta_2}(t) = L_1\circ\varrho\circ L\circ\varrho\circ L\circ\varrho\circ L\circ\text{Fourier}(t),
$$
where $\text{Fourier}(\cdot)$ is the Fourier embedding function with $128$ outputs, $L$ is a linear dense layer with $64$ outputs, $L_1$ is a linear dense layer with single output and $\varrho$ is the GELU \cite{hendrycks_gaussian_2023} activation function. We mention that this neural network architecture is similar to the one used in \cite{berner_optimal_2023}. However, the above neural network implements a scalar function that approximates the solution to a PDE, whereas the neural network in  \cite{berner_optimal_2023} approximates the gradient of the solution to the PDE.

For a full interpolant-based sampler, as explained earlier, we need to solve two PDEs. This requires two neural networks--one for $t\in[0,T']$ and another for $t\in[T',T]$. The networks we use are as follows:
$$
u^\theta(t,x) = \Phi_1^{\theta_1}(t,x)+\Phi_2^{\theta_2}(t)(-\psi(T',x)),
$$
$$
v^{\theta'}(t,x) = \Phi_1^{\theta_1'}(t,x)+\Phi_2^{\theta_2'}(t)\varphi(x),
$$
where $\psi$ is the function that appear in (\ref{eqn:velocity}) and $\varphi$ is given in (\ref{eqn:score_term_condition}.

\subsection{Training}
We use the Adam optimizer \cite{kingma_adam_2017} to optimize the weights of the neural networks. We use a piecewise linear scheduling for the learning rate with its value becoming 1/5th every 2000 training steps. We also do exponential moving averaging of the neural network weights. We detached the process \ref{eqn:fbsde_init_process} from the computational graph while computing the gradients. Table~\ref{tab:hyperparams} shows the values of the training-related hyperparameters used in our experiments.

\begin{table}
  \caption{Hyperparameters}
  \label{tab:hyperparams}
  \centering
  \begin{tabular}{ll}
    \toprule
    \multicolumn{2}{c}{\textbf{Sampler}} \\
    $T$ & 1.0\\
    $T'$ & 0.5\\
    $\beta(t)$ & $\frac{r(t)}{g(t)}$\\
    $\alpha(t)$ & 1.0\\
    $\delta$ & $5*10^{-6}$\\
    $\lambda$ & $2000/\delta$\\
    \midrule
    \multicolumn{2}{c}{\textbf{Training}}\\
    Number of steps & 10000\\
    Batch Size & 128\\
    Initial Learning rate     & $5*10^{-3}$\\
    Max Grad Norm   & 1.0\\
    SDE Steps & 60\\
    \midrule
    \multicolumn{2}{c}{\textbf{Sampling}}\\
    SDE steps & 1000\\
    Samples & 10000\\
    $\eps$ & 1.0\\
    \bottomrule
  \end{tabular}
\end{table}
Figure~\ref{fig:trainphase} shows the training loss as a function of training steps. The plot on the left is for GMM with $d=2$, and the plot on the right is for MoG ($d=100$).
\begin{figure}[!htb]
    \centering
        \includegraphics[scale=0.5]{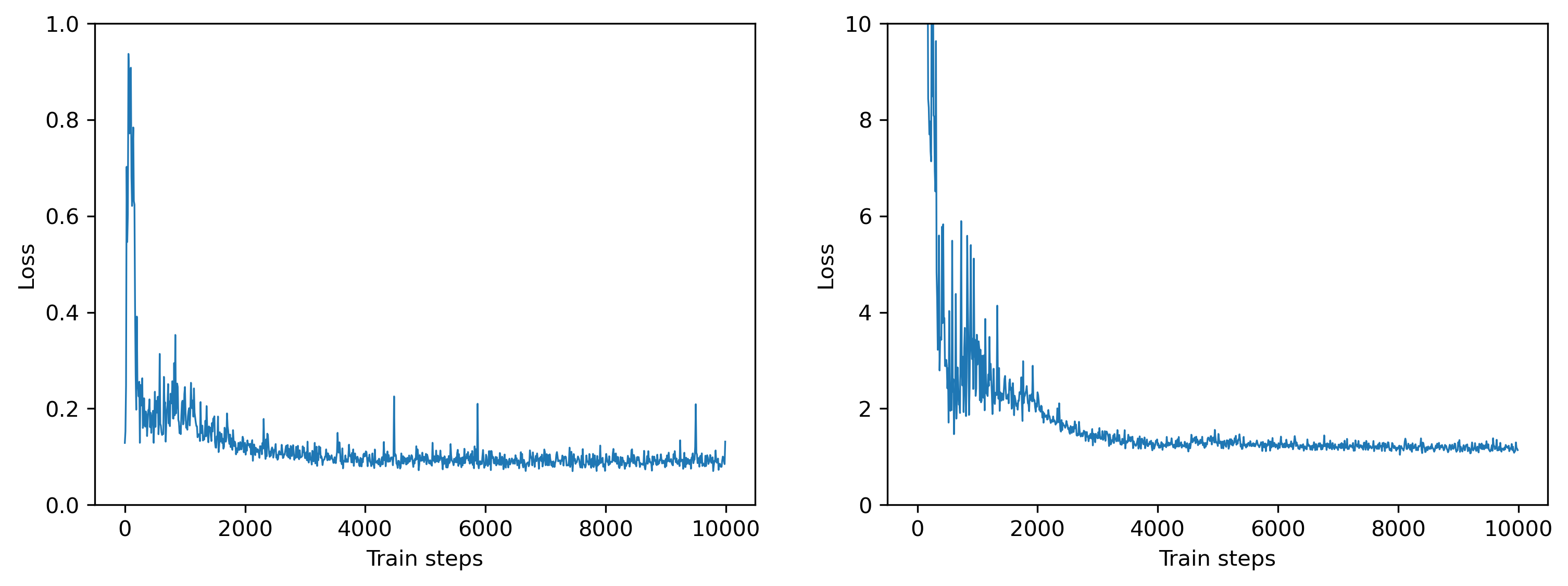}
        \caption{Training phase plot at $d=2$ and $d=100$}
        \label{fig:trainphase}
\end{figure}

\subsection{Sampling}
We sample using procedures given in \ref{proc:sample_half_int} and \ref{proc:smaple_int}. The number of discretization steps we use is $1000$. We use a constant diffusion coefficient while sampling. See Table~\ref{tab:hyperparams}. 

\section{Estimating the normalization constant}\label{apndx:est_logZ}
One of the remarkable features of diffusion-based sampling approaches is that they allow for estimating the normalization constant of the given unnormalized density. We describe here the method for estimating the normalization constant for the half interpolant-based sampler. The log normalization constant can be estimated similarly for a full interpolant-based sampler. Let $u$ be the function given in (\ref{eqn:velocity}). We know $u$ is the solution to PDE (\ref{eqn:velocity_pde}) under the terminal condition \ref{eqn:terminal_cond_vel_pde}. Let $X_t$ be the solution to an SDE given by
\begin{equation*}
    dX_t = \left(\sigma^2(t)\nabla u(t,X_t)+ \mu(t,X_t)\right)dt + \sigma(t)dW_t,\quad X_0=0.
\end{equation*}
Applying It\^o's lemma to $u(t,X_t)$, we have
\begin{align*}
    du(t,X_t) &= \partial_t u(t,X_t)dt + \nabla u(t,X_t)^TdX_t + \frac{1}{2}\sigma^2(t)\Delta u(t,X_t)dt\\
    &=  \left(\partial_t u+\frac{1}{2}\sigma^2(t)\Delta u+\left(\sigma^2(t)\nabla u +\mu(t,X_t)\right)^T\nabla u\right)dt + \sigma(t)\nabla u^TdW_t\\
    &=  \frac{1}{2}\sigma^2(t)\norm{\nabla u(t,X_t)}^2dt + \sigma(t)\nabla u^TdW_t\\
\end{align*}
Integrating from $0$ to $T$, we get
\begin{equation}\label{eqn:logZ_est}
    u(0,0) = u(T,X_T)-\frac{1}{2}\int_0^T\sigma^2(t)\norm{\nabla u(t,X_t)}^2dt-\int_0^T\sigma(t)\nabla u^TdW_t
\end{equation}
Equation (\ref{eqn:logZ_est}) gives a viable strategy to estimate the normalization constant. Suppose we want to sample from $\pi = \frac{\hat{\pi}}{Z}$, given only $\hat{\pi}$. From the definition of $u$, we have $u(T,X_T) = \log\frac{\pi(\beta(T)X_T)}{\psi(T,\beta(T)X_T)} = \log\frac{\hat{\pi}(\beta(T)X_T)}{\psi(T,\beta(T)X_T)} - \log Z $. Moreover, we have $u(0,0) = 0$. Hence, we get
\begin{equation*}
0 = -\frac{1}{2}\int_0^T\sigma^2(t)\norm{\nabla u(t,X_t)}^2dt-\int_0^T\sigma(t)\nabla u^TdW_t + \log\frac{\hat{\pi}(\beta(T)X_T)}{\psi(T,\beta(T)X_T)} - \log Z,
\end{equation*}
which after taking expectation gives
\begin{multline}
    \log Z = \mathbb{E}\bigg[-\frac{1}{2}\int_0^T\sigma^2(t)\norm{\nabla u(t,X_t)}^2dt-\int_0^T\sigma(t)\nabla u^TdW_t\\ +\log\hat{\pi}(\beta(T)X_T)-\log(\psi(T,\beta(T)X_T))\bigg]
\end{multline}

\section{Additional Experiments and Results}\label{apndx:additionalExps}
\begin{figure}
  \centering
  \includegraphics[scale = 0.55]{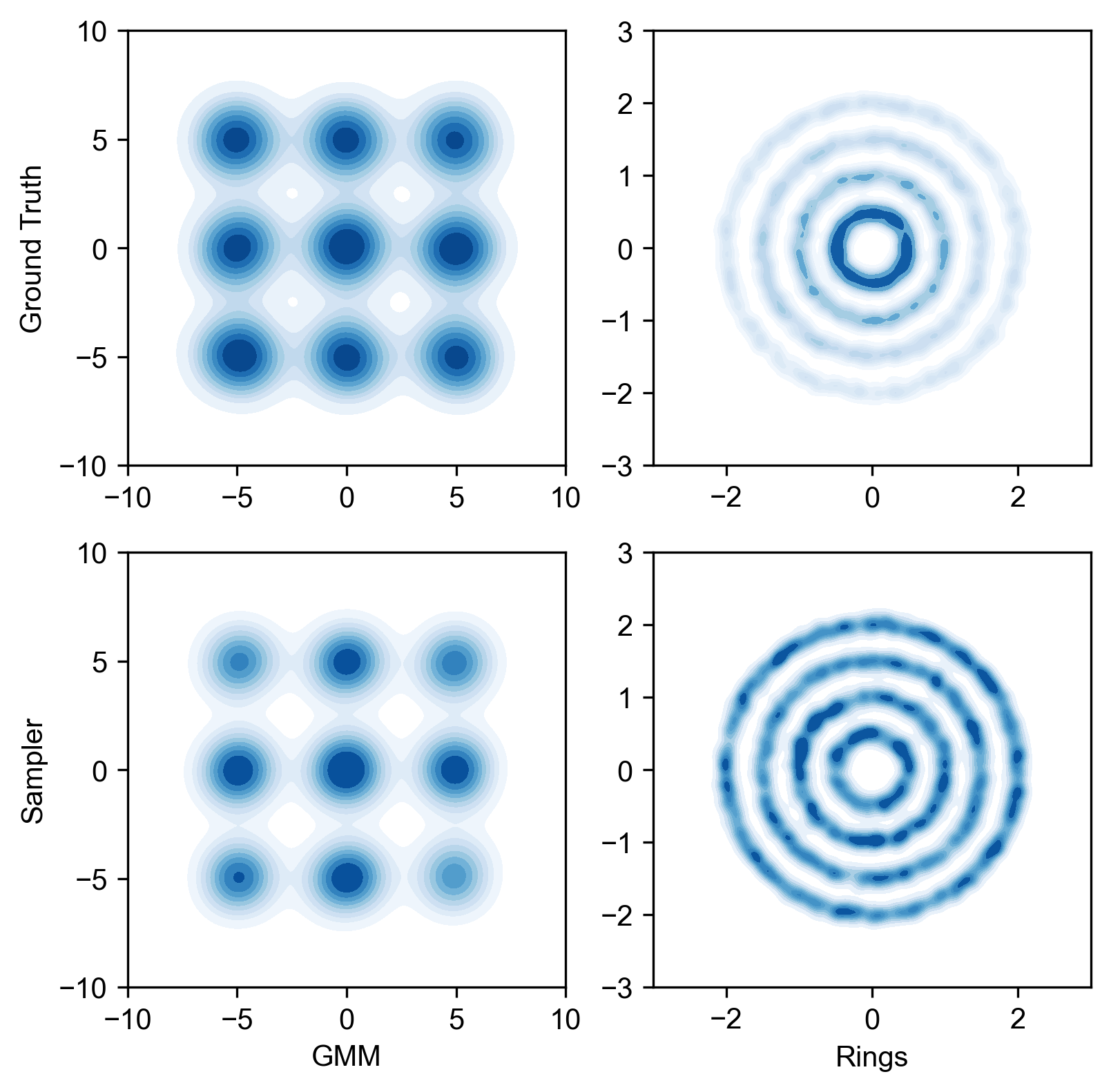}
  \caption{KDE plots of Gaussian mixture, and Rings distributions.}
  \label{fig:kde_plots_tech}
\end{figure}
Figure~\ref{fig:kde_plots_tech} shows the kernel density estimation (KDE) plots of samples obtained by our methods for 2-dimensional Gaussian mixture distribution and Rings distribution. These plots illustrate that our method can sample from distributions with multiple modes. 
Figure~\ref{fig:estimates_functionals} shows the estimates of $\mathbb{E}|X|_1$ (Mean Absolute) and $\mathbb{E} \norm{X}^2$ (Mean Squared) computed from the samples obtained using full interpolant based sampler with different interpolants. Refer to Table~\ref{tab:estimates_tech} for the estimates at the end of the training for the interpolant $g(t)=\text{sin}(\pi t/2), r(t)=\text{cos}(\pi t/2)$. These estimates are without correction using important weights and hence indicate the true quality of the samples.
\begin{figure}
  \centering
  \includegraphics[scale = 0.35]{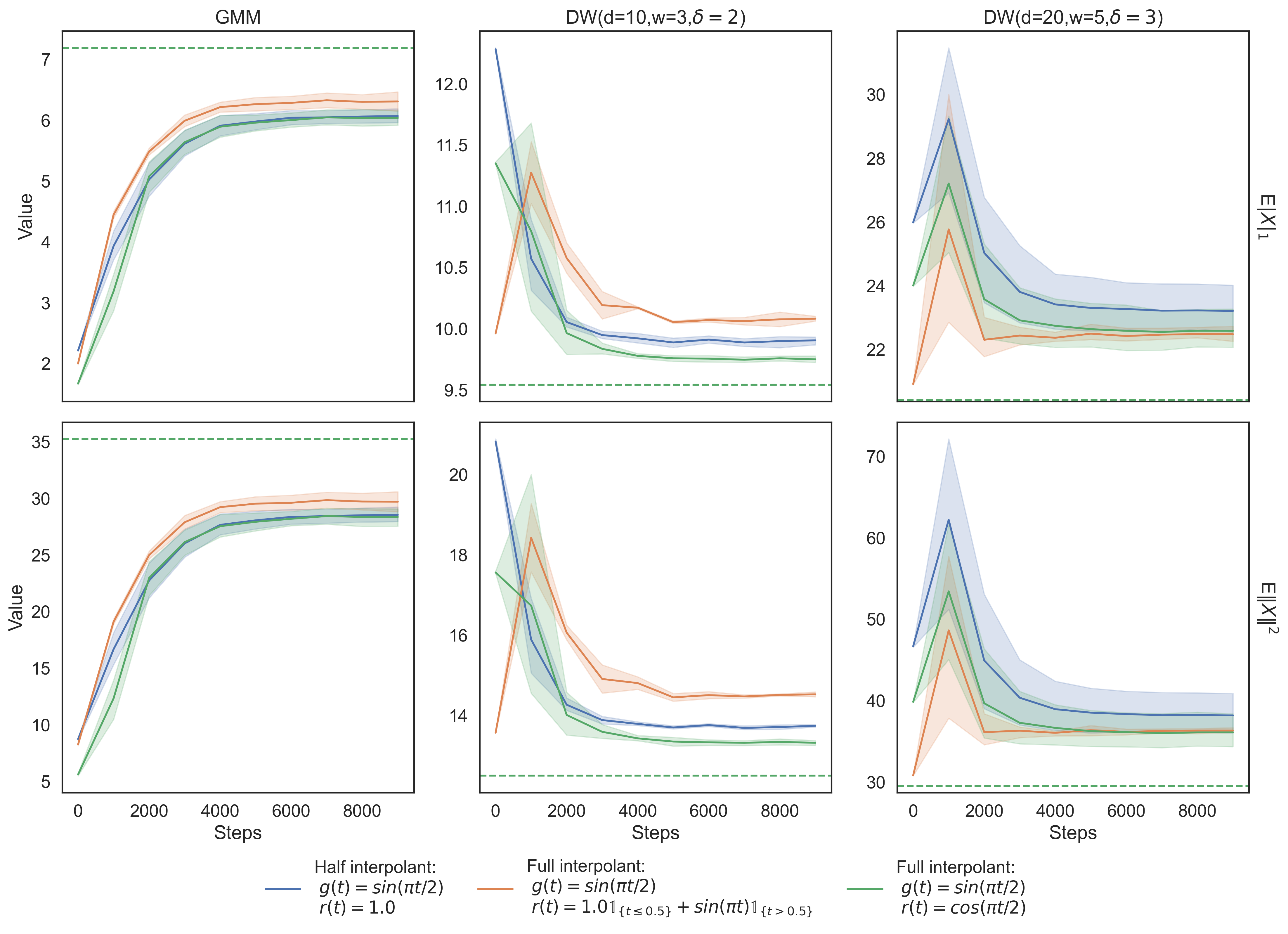}
  \caption{Estimates of $\mathbb{E}|X|_1$ and $\mathbb{E}\norm{X}^2$ as a function of training steps.}
  \label{fig:estimates_functionals}
\end{figure}

\begin{table}
  \caption{Estimates (True values in the paranthesis)}
  \label{tab:estimates_tech}
  \centering
  \begin{tabular}{llll}
    \toprule
     \multirow{2}{*}{Quantity} & \multicolumn{3}{c}{Distributions} \\
     \cmidrule(r){2-4}
                               &    GMM          & DW$(d=10,w=3,\delta=2)$   & DW$(d=20,w=5,\delta=3)$  \\
     \midrule
             $\mathbb{E}{|X|_1}$     &    6.04$\pm$0.18 (7.19) & 9.75$\pm$0.04 (9.54)  &22.59$\pm$0.87 (20.42)    \\
              $\mathbb{E}{\norm{X}^2}$     &  28.36$\pm$1.05 (35.26) & 13.32$\pm$0.08 (12.51)  &36.11$\pm$2.88 (29.54)  \\
    \bottomrule
  \end{tabular}
\end{table}

\subsection{Sampling from a statistical physics inspired high-dimensional distribution}\label{sec:gaussian_spin_glass}
Now, we consider a statistical physics \textit{spin glass} model \cite{arous_aging_2001, barra_about_2014}. The precise definition is slightly different in these two works and we follow the incarnation of \cite{barra_about_2014}. This model was introduced as a ``soft spin'' version of the celebrated {\it spherical} Sherrington-Kirkpatrick model \cite{kosterlitz_spherical_1976}. The free energy can be computed exactly and here serves as a benchmark to assess the experiments. 

The probability density of this soft spin spherical spin glass is given by
\begin{equation*}
    \pi(x) = \frac{1}{Z}\exp\left\{\frac{\beta}{\sqrt{2d}}\sum_{i,j=1}^dA_{ij}x_ix_j-\frac{\beta^2}{4d}\left(\sum_{i=1}^dx_i^2\right)^2-\frac{1}{2}\sum_{i=1}^dx_i^2\right\},
\end{equation*}
where $A_{ij}\overset{\text{i.i.d.}}{\sim}\cN{0,I_d}$. The asymptotic rigorous prediction for the free energy \cite{barra_about_2014} is given by
$$
\lim_{d\rightarrow\infty}\frac{1}{d}\mathbb{E}\log Z = \begin{cases}
    0 &\text{for}\;\beta < 1,\\
    -\frac{1}{2}\log\beta + \frac{\beta\bar{q}}{2}+\frac{\beta^2\bar{q}^2}{d},\quad \bar{q} = \frac{\beta-1}{\beta^2} &\text{for}\;\beta \ge 1
\end{cases}
$$
where the limit is attained almost surely, in other words it is equal to the limit of the expectation with respect to $A_{ij}\overset{\text{i.i.d.}}{\sim}\cN{0,I_d}$).
The theory predicts that there is a phase transition at $\beta=1$.

We use our full interpolant-based sampler to estimate $\log Z$ for different values of $\beta$. In Figure~\ref{fig:logZ_gsg} we compare the obtained estimate of $\frac{1}{d}\log Z$ with the theoretical prediction. The estimates obtained through our sampler are consistent with the phase transition. The experimental points correspond to one instance of teh coupling matrix $A_{ij}$ indicating that already for $d=100$ the free energy is well concentrated around its average. 

\begin{figure}
  \centering
  \includegraphics[scale = 0.7]{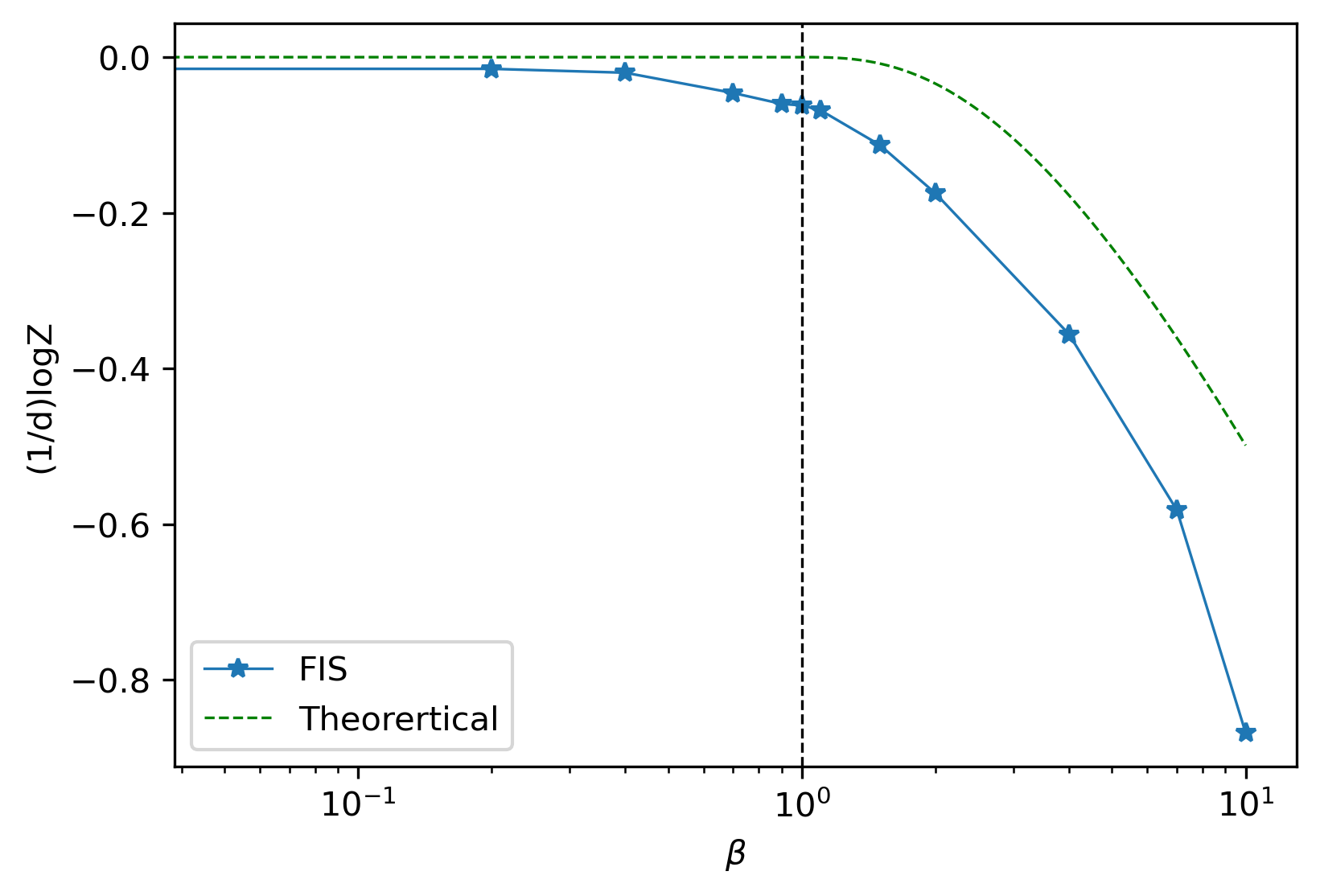}
  \caption{Estimate of $\frac{1}{d}\log Z$ for Gaussian Spin Glass model with $d=100$.}
  \label{fig:logZ_gsg}
\end{figure}

\subsection{Training and Sampling Times}
\begin{table}
  \caption{Training and sampling times}
  \label{tab:times}
  \centering
  \begin{tabular}{LLL}
    \toprule
     Distribution & Training time \par per step (mS) & Sampling time per \par discretization step (mS) \\
     \midrule
    GMM & 97.7 & 9.3\\
    Funnel & 99.7 & 10.3\\
    Double well ($d=10$) & 102.7 & 9.8\\
    Double well ($d=20$) & 105.0 & 10.8\\
    Spin Glass ($d=100$) & 98.1 & 10.4\\
    \bottomrule
  \end{tabular}
\end{table}
In this section, we discuss the typical training and sampling time for our sampler. We ran our experiments on an NVIDIA GeForce GTX 1070 (8GB) GPU. In comparison to the conventional MCMC algorithms, a disadvantage of our methods is that we have learnable parameters. However, the training time required can be amortized in time if we are interested in drawing a large number of samples from a specific target distribution. Table~\ref{tab:times} shows the training time per gradient descent step required for different distributions. We take a total of 10,000 gradient descent steps during training. Hence, the total time required for training is around 17 minutes for each of the distributions. Table~\ref{tab:times} also shows the time required during the sampling phase per discretization step of the SDE. The number of Euler-Maruyama discretization steps we take is 1000, which amounts to a total of 10 seconds for drawing 10,000 samples. 

\subsection{Comparison to Langevin Diffusion for Sampling}
\begin{figure}
  \centering
  \includegraphics[scale = 0.55]{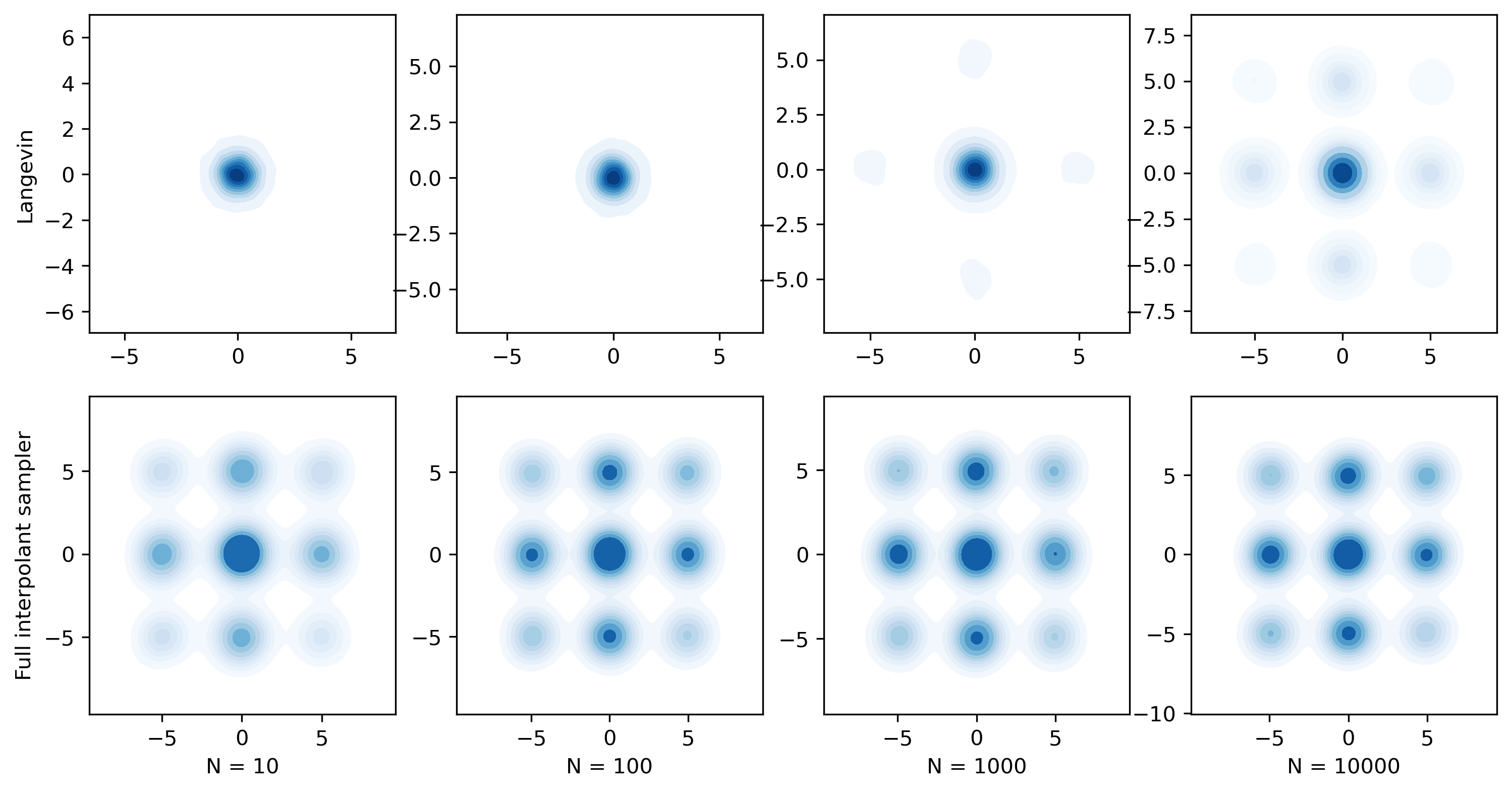}
  \caption{KDE plots of GMM samples generated by Langevin diffusion-based algorithm and full interpolant-based algorithm for different number of discretization steps.}
  \label{fig:langevin}
\end{figure}

\begin{table}
  \caption{Comparison of estimates using Langevin and Full interpolant sampler.}
  \label{tab:langevin_estimates}
  \centering
  \begin{tabular}{cccccc}
    \toprule
     Distribution & Steps & \multicolumn{2}{c}{$\mathbb{E}|X|_1$} & \multicolumn{2}{c}{$\mathbb{E}\norm{X}^2$} \\
     \cmidrule(r){3-6}
     & & LMC & FIS & LMC & FIS\\
     \midrule
    \multirow{4}{*}{GMM} & 10 & 1.1 & 5.5 & 1.5 & 25.2\\
    & 100 & 1.1 & 5.9 & 1.7 & 28.0\\
    & 1000 & 1.6 & 6.0 & 4.4 & 28.0\\
    & 10000 & 4.7 & 6.0 & 21.3 & 28.4\\
    \midrule
    \multirow{4}{*}{DW($d=10$)} & 10 & 8.1 & 9.7 & 9.9 & 13.3\\
    & 100 & 9.0 & 9.72 & 11.42 & 13.2\\
    & 1000 & 9.7 & 9.5 & 12.5 & 13.3\\
    & 10000 & 9.7 & 9.5 & 12.5 & 13.3\\
    \midrule
    \multirow{4}{*}{DW($d=20$)} & 10 & 16.4 & 22.0 & 20.6 & 22.15\\
    & 100 & 19.5 & 22.2 & 27.3 & 34.6\\
    & 1000 & 20.5 & 22.2 & 29.6 & 34.6\\
    & 10000 & 20.4 & 22.1 & 29.4 & 34.5\\
    \bottomrule
  \end{tabular}
\end{table}

To illustrate the benefits of our method, we compare it to the vanilla Langevin diffusion-based sampling algorithm. A Langevin diffusion is given by the solution to the SDE:
\begin{equation*}
    dS_t = \eta\nabla\log\pi(S_t)dt + \sqrt{2\eta} dW_t.
\end{equation*}
 It is well known that the distribution of $S_t$ as $t$ goes to infinity is $\pi$. A sampling scheme based on Langevin diffusion is given by its discretization:
 \begin{equation*}
     S_{n+1} = S_n + \nabla\log\pi(S_t)\delta + \sqrt{2\delta}w_{n+1},
 \end{equation*}
 where $w_i\overset{i.i.d.}{\sim}\cN{0,I_d}$. In Figure~\ref{fig:langevin} we compare the KDE plots of samples generated by a Langevin diffusion-based algorithm ($\delta=0.1$) and a full interpolant-based algorithm as a function of the number of discretization steps for a Gaussian mixture model. In this experiment, we made sure that the time required per discretization step is the same for both algorithms. We observe that the Langevin diffusion requires a large number of steps to explore all the modes of the distribution, while our algorithm explores all the modes even with merely 10 discretization steps. A comparison of estimates of $\mathbb{E}|X|_i$ and $\mathbb{E}\|X\|^2$ of Gaussian mixture and double well distributions obtained through Langevin Monte Carlo (LMC) and Full interpolant sampler (FIS) is given in Table~\ref{tab:langevin_estimates}.
\subsection{Comparison with Generalized Bridge Sampler (GBS)}\label{sec:comparison_gbs}
\begin{figure}
  \centering
  \includegraphics[scale = 0.4]{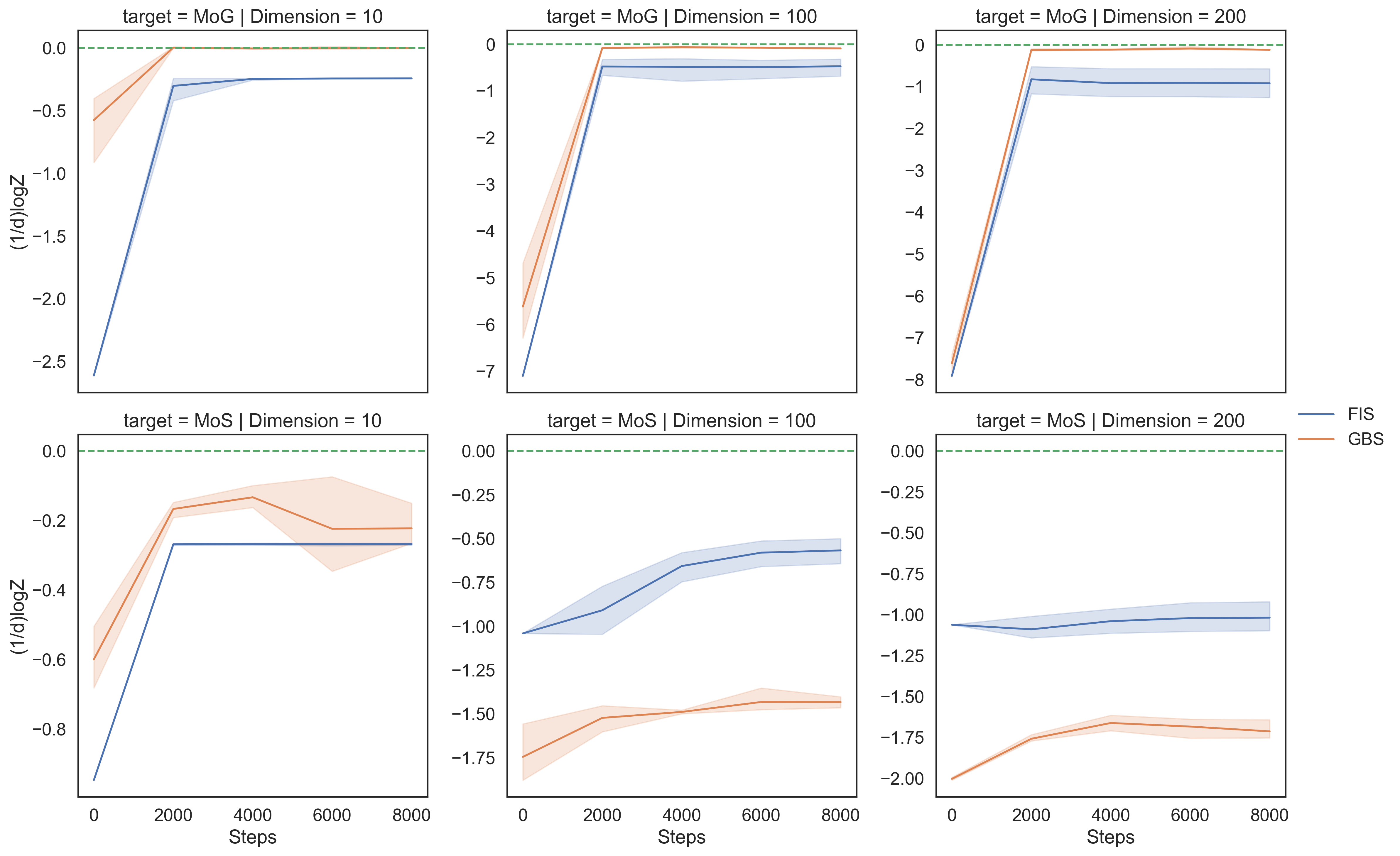}
  \caption{Estimates of $\frac{1}{d}\log Z$ as a function of training steps given by FIS and GBS for different targets and dimensions. }
  \label{fig:gbs_comp_plot}
\end{figure}

In Figure~\ref{fig:gbs_comp_plot} we compare the performance of FIS (Full Interpolant Sampler with trigonometric interpolant) and GBS (General Bridge Sampler) \cite{richter_improved_2023, blessing_beyond_2024} as a function of training steps for two different target distributions and for dimensions $d=10,100,200$. Here, MoG refers to a Mixture of isotropic Gaussian distribution with 10 components with their mean selected uniformly at random from $[-10,10]^d$ and MoS is a Mixture of Student t-distribution (degree of freedom = 2) with 10 components with their mean selected uniformly at random from $[-10,10]^d$. We chose to compare our algorithm with GBS because of their similar computational complexity. The implementation of GBS was taken from \cite{blessing_beyond_2024}.


\begin{table}[!htb]
  \caption{Estimates of $\log Z$ and 2-Wasserstein distance obtained using FIS and GBS for different targets and dimensions.}
  \label{tab:gbs_com_was}
  \centering
  \begin{tabular}{cccccc}
    \toprule
     Distribution & Dimension & \multicolumn{2}{c}{$\log Z$ (True value = 0)} & \multicolumn{2}{c}{2-Wasserstein} \\
     \cmidrule(lr){3-4} \cmidrule(lr){5-6}
     & & FIS & GBS & FIS & GBS \\
     \midrule
    \multirow{3}{*}{MoG} & $d=10$ & $-2.43 \pm 0.02$ & $-0.02 \pm 0.03$ & $243.5 \pm 5$ & $6.21 \pm 0.04$ \\
     & $d=100$ & $-47.24 \pm 26.1$ & $-8.96 \pm 1.25$ & $5088 \pm 688$ & $147.8 \pm 0.83$ \\
     & $d=200$ & $-183.56 \pm 90.24$ & $-23.91 \pm 1.26$ & $10035 \pm 1136$ & $323.4 \pm 2.3$ \\
     \midrule
    \multirow{3}{*}{MoS} & $d=10$ & $-2.68 \pm 0.04$ & $-2.22 \pm 0.63$ & $285.9 \pm 110$ & $317.1 \pm 8.8$ \\
     & $d=100$ & $-56.75 \pm 9.5$ & $-143.3 \pm 3.11$ & $3156 \pm 1253$ & $14932 \pm 76$ \\
     & $d=200$ & $-203.77 \pm 23.65$ & $-342.7 \pm 12.24$ & $9208 \pm 1653$ & $34573 \pm 3650$ \\
    \bottomrule
  \end{tabular}
\end{table}

Table~\ref{tab:gbs_com_was} compares the performance of FIS and GBS at the end of training (10000 iterations). The experimental setting is the same as that of Figure~\ref{fig:gbs_comp_plot}. The metrics we use for comparison are 
 and 2-Wasserstein distance (Entropy regularized optimal transport cost) \cite{blessing_beyond_2024}. Based on the observations in Figure~\ref{fig:gbs_comp_plot} and Table~\ref{tab:gbs_com_was}, it is not possible to definitively conclude that one method is better than the other, as their performance may depend on the specific distribution being considered.
 \subsection{Comparison with other Diffusion-based Samplers}
 Existing diffusion-based samplers that could be considered closest to our work are PIS \cite{zhang_path_2022}, DIS \cite{berner_optimal_2023}, and DDS \cite{vargas_denoising_2022}. These methods facilitates the use of important sampling techniques to correct the samples while estimating funcitonals. The estimate of $\log Z$ of GMM without importance sampling is mentioned in \cite{zhang_path_2022}, and it has a bias of -0.44. In this setting, our method has a lower bias given by $-0.26$. The $\log Z$ estimate of the Funnel distribution reported in \cite{zhang_path_2022} is much better than ours. However, as noted in \cite{berner_optimal_2023} and \cite{vargas_denoising_2022}, the Funnel distribution used in \cite{zhang_path_2022} has a favourable variance compared to the standard Funnel distribution that we used.

 Training of neural networks in \cite{zhang_path_2022}, DIS \cite{berner_optimal_2023}, and DDS \cite{vargas_denoising_2022} involves computing gradient of a Neural SDE, which can turn quite expensive computationally. However, our methods do not take gradient with respect to the process (\ref{eqn:fbsde_init_process}).  

 \subsection{Effect of different parameters}\label{sec:ablation}
 Table~\ref{tab:ablation} shows the influence of different hyperparameters on the estimates. The target distribution used is GMM ($d=2$).
 \begin{table}[!htb]
\caption{Influence of different hyperparameters on the estimates obtained. Here, $c_1 = 5*10^{-6}$, $c_2=4*10^{8}$.}
    \centering
    \begin{tabular}{ccccc}
        \toprule
         Hyperparameter & Value & \multicolumn{3}{c}{Estimates}\\
         \cmidrule(r){3-5}
         & & $\log Z$ & $\mathbb{E}|X|_1$ & $\mathbb{E}\norm{X}^2$\\
         \midrule
        \multirow{4}{*}{$\delta$} & $0.1c_1$ & $-0.66 \pm 0.10$ & $4.69 \pm 0.39$ & $20.44 \pm 2.26$\\
        & $1.0c_1$ & $-0.25 \pm 0.02$ & $6.05 \pm 0.24$ & $28.42 \pm 1.37$\\
        & $10c_1$ & $-0.67 \pm 0.19$ & $4.27 \pm 0.71$ & $18.41 \pm 3.93$ \\
        \midrule
        \multirow{4}{*}{$\lambda$} & $0.1c_2$ & $-0.30 \pm 0.02$ & $5.85 \pm 0.17$ & $27.25 \pm 0.97$\\
        & $1c_2$ & $-0.26 \pm 0.01$ & $6.02 \pm 0.21$ & $28.28 \pm 1.18$\\
        & $10c_2$ & $-0.74 \pm 0.16$ & $4.1 \pm 0.55$ & $17.4 \pm 3.05$ \\
        \midrule
        \multirow{4}{*}{SDE steps} & 30 & $-0.31 \pm 0.04$ & $5.86 \pm 0.31$ & $27.45 \pm 1.77$\\
        & 60 & $-0.27 \pm 0.02$ & $6.05 \pm 0.22$ & $28.47 \pm 1.22$\\
        & 100 & $-0.27 \pm 0.03$ & $6.07 \pm 0.20$ & $28.55 \pm 1.17$ \\
        \bottomrule
      \end{tabular}
    \label{tab:ablation}
\end{table}

\end{document}